\documentclass[lettersize,journal]{IEEEtran}
\usepackage{amsmath,amsfonts}
\usepackage{algorithmic}
\usepackage{algorithm}
\usepackage{array}
\usepackage[caption=false,font=footnotesize,labelfont=rm,textfont=rm]{subfig}
\usepackage{textcomp}
\usepackage{stfloats}
\usepackage{url}
\usepackage{verbatim}
\usepackage{graphicx}
\usepackage{cite}
\usepackage{bm}
\usepackage{multirow}
\usepackage{booktabs}
\usepackage{amsthm}
\usepackage{makecell}
\usepackage{cuted}

\newtheorem{theorem}{Theorem}

\begin{document}

\title{PointWavelet: Learning in Spectral Domain for 3D Point Cloud Analysis}

\author{Cheng Wen, Jianzhi Long, Baosheng Yu, Dacheng Tao,~\IEEEmembership{Fellow,~IEEE}
\thanks{Cheng Wen, Jianzhi Long, Baosheng Yu, and Dacheng Tao are with the School of Computer Science at the University of Sydney, NSW 2008, Australia. Emails: \{cwen6671; jlon7198\}@uni.sydney.edu.au, \{baosheng.yu; dacheng.tao\}@sydney.edu.au.  }
\thanks{Manuscript received April 19, 2021; revised August 16, 2021.}
}



\maketitle

\begin{abstract}
With recent success of deep learning in 2D visual recognition, deep learning-based 3D point cloud analysis has received increasing attention from the community, especially due to the rapid development of autonomous driving technologies.
However, most existing methods directly learn point features in the spatial domain, leaving the local structures in the spectral domain poorly investigated. In this paper, we introduce a new method, PointWavelet, to explore local graphs in the spectral domain via a learnable graph wavelet transform. Specifically, we first introduce the graph wavelet transform to form multi-scale spectral graph convolution to learn effective local structural representations. To avoid the time-consuming spectral decomposition, we then devise a learnable graph wavelet transform, which significantly accelerates the overall training process. Extensive experiments on four popular point cloud datasets,  ModelNet40, ScanObjectNN, ShapeNet-Part, and S3DIS, demonstrate the effectiveness of the proposed method on point cloud classification and segmentation.
\end{abstract}


\section{Introduction}

\IEEEPARstart{W}{ith} the development of 3D sensing technologies, large-scale 3D datasets are now increasingly available, such as ModelNet~\cite{wu20153d}, ShapeNet~\cite{chang2015shapenet}, ScanObjectNN~\cite{uy2019revisiting}, S3DIS~\cite{armeni20163d} and KITTI~\cite{geiger2012we}, which significantly help machines to better perceive and understand the real world we live in. When referring to 3D data, the main formats include point cloud~\cite{qi2017pointnet}, voxel~\cite{zhou2018voxelnet}, mesh~\cite{wang2018pixel2mesh} and multi-view images~\cite{su2015multi}.
Recently, point cloud has received increasing attention from the community considering that it is usually the first-hand data captured by popular scanning devices such as Kinect, PrimeSense and LiDAR. Therefore, it has been extensively applied in various scene-perception fields such as scene reconstruction~\cite{lan2019robust,hu2018semantic}, autonomous driving~\cite{nagy2018real}, and virtual reality~\cite{wirth2019pointatme}. 

To learn effective representations, a point cloud can be processed in either the spatial or the spectral domain. Specifically, spatial learning methods such as PointNet++~\cite{qi2017pointnet++}, PointCNN~\cite{li2018pointcnn} and DGCNN~\cite{wang2019dynamic} directly utilize spatial convolution/pooling operations to extract point features. However, the irregular and unordered spatial structures of point clouds pose a challenge to existing methods, which prompts us to learn representations in the spectral domain. Though successful for 1D time series data and 2D images, classical frequency analysis struggles to learn effectively from irregularly distributed data such as molecules, point clouds and social media networks.

\begin{figure*}[!ht]
\centering
    \subfloat[]{
        \begin{minipage}[b]{0.18\linewidth} 
        \includegraphics[width=1\textwidth]{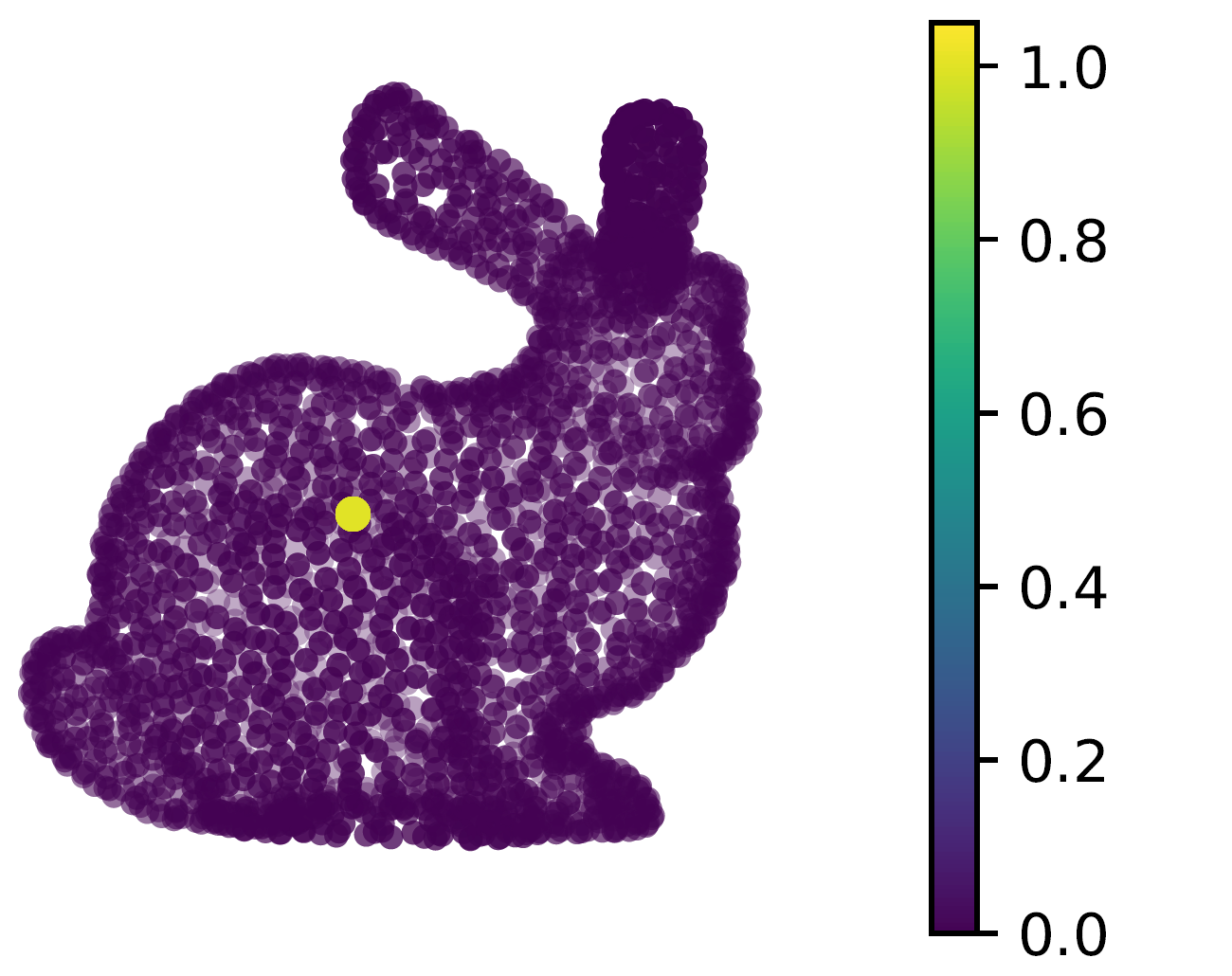}
        \end{minipage}
        \label{fig.localization.input}
    }
    \subfloat[]{
        \begin{minipage}[b]{0.18\linewidth} 
        \includegraphics[width=1\textwidth]{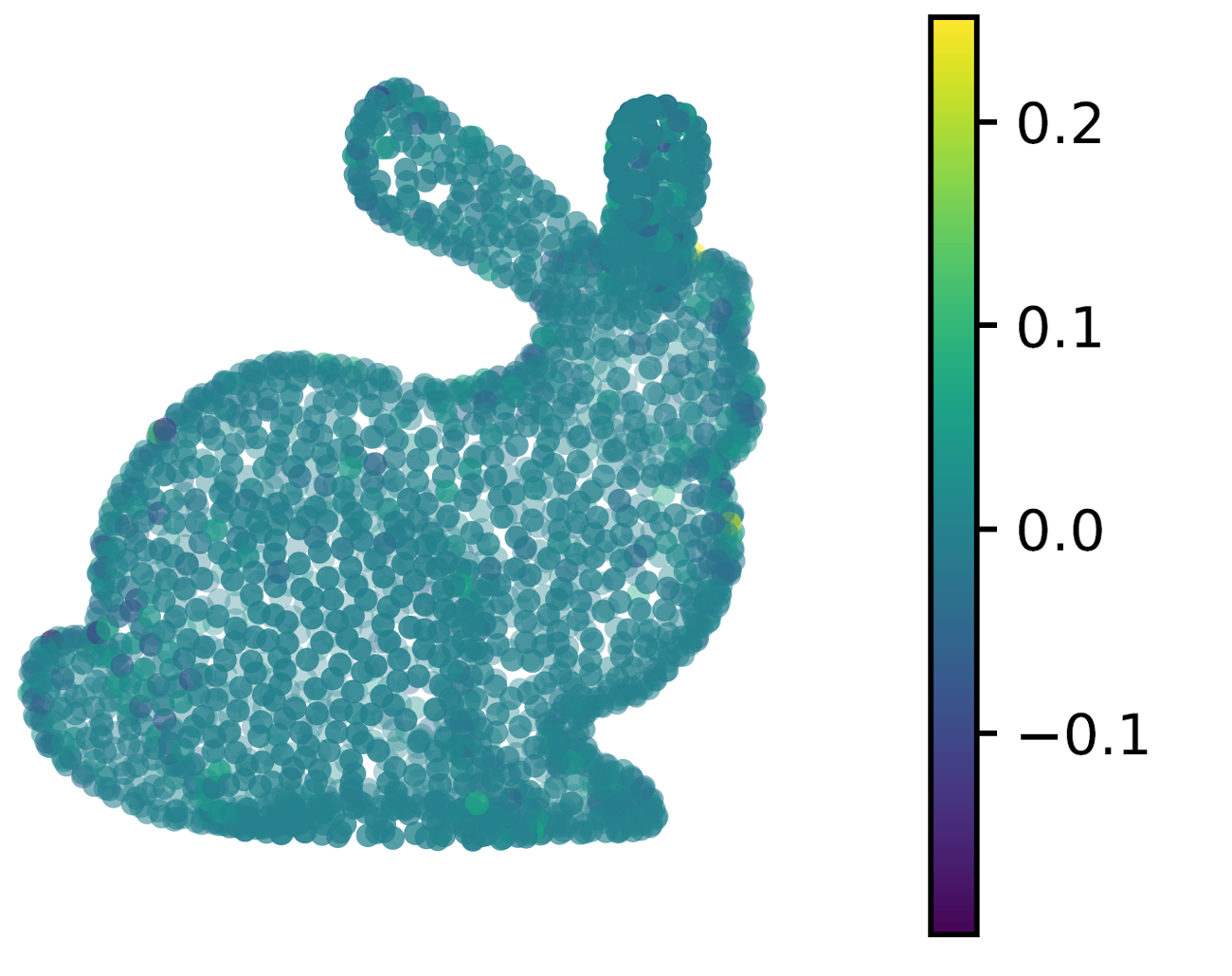}
        \end{minipage}
    }
    \subfloat[]{
        \begin{minipage}[b]{0.18\linewidth} 
        \includegraphics[width=1\textwidth]{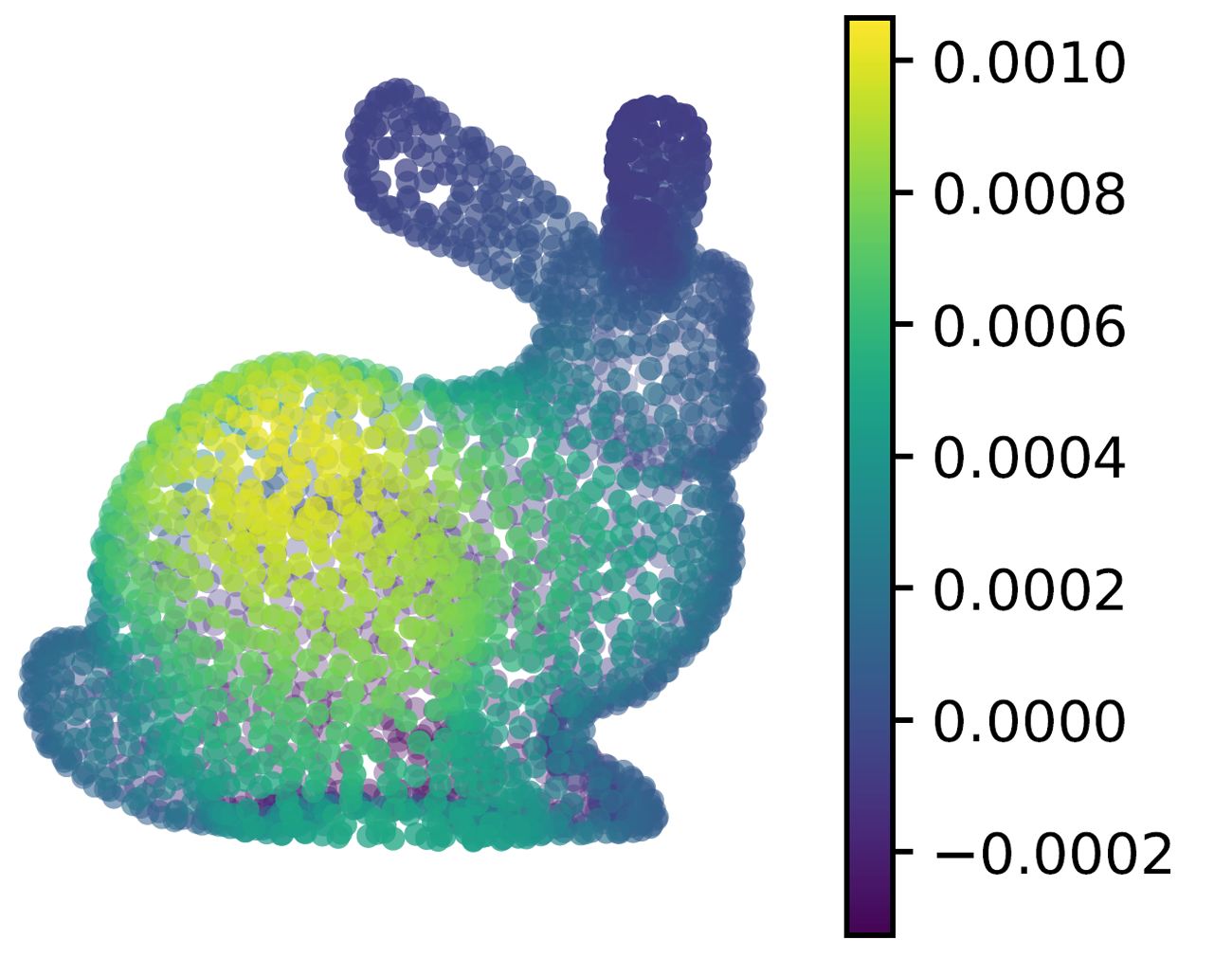}
        \end{minipage}
    }
    \subfloat[]{
        \begin{minipage}[b]{0.18\linewidth} 
        \includegraphics[width=1\textwidth]{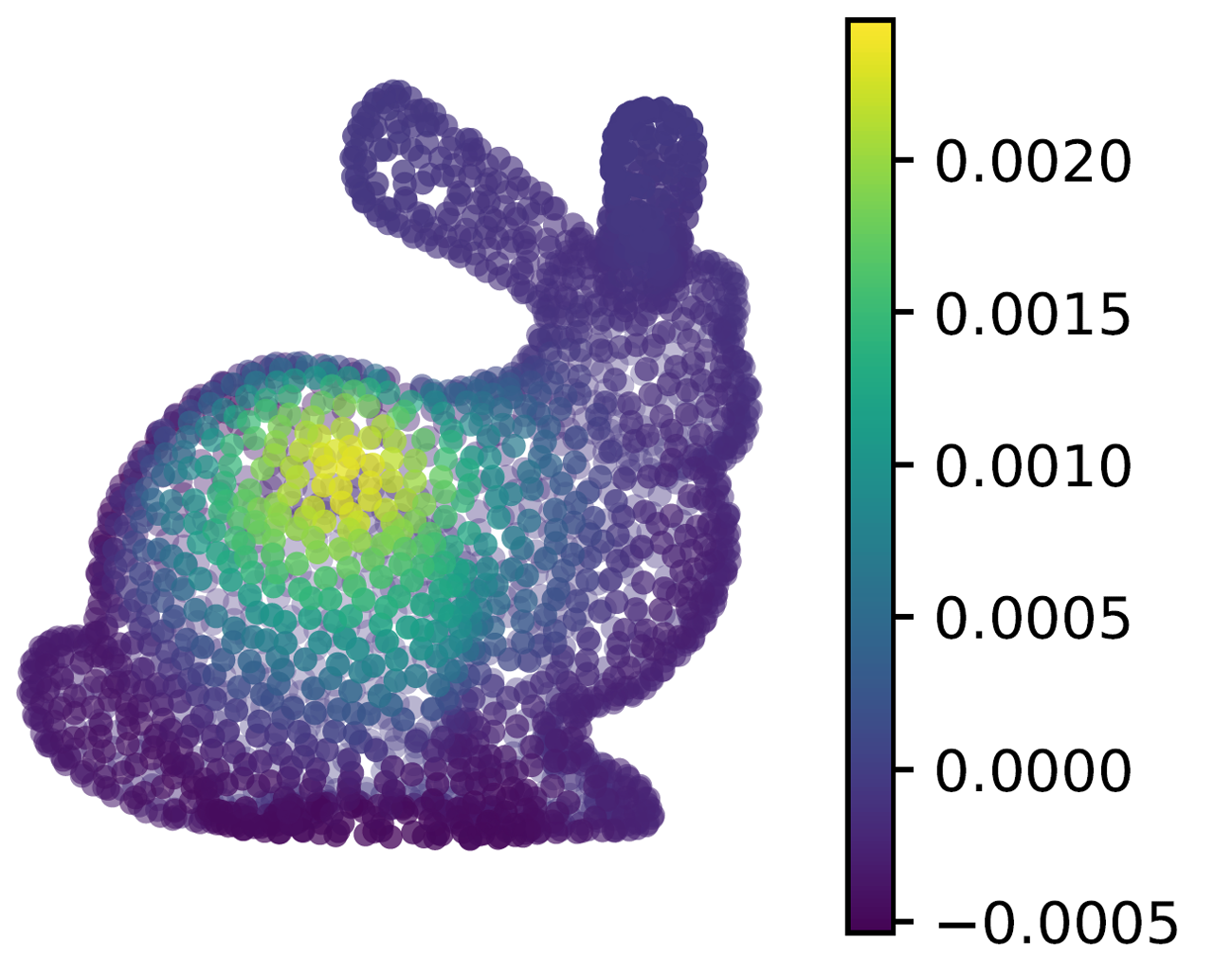}
        \end{minipage}
    }
    \subfloat[]{
        \begin{minipage}[b]{0.18\linewidth} 
        \includegraphics[width=1\textwidth]{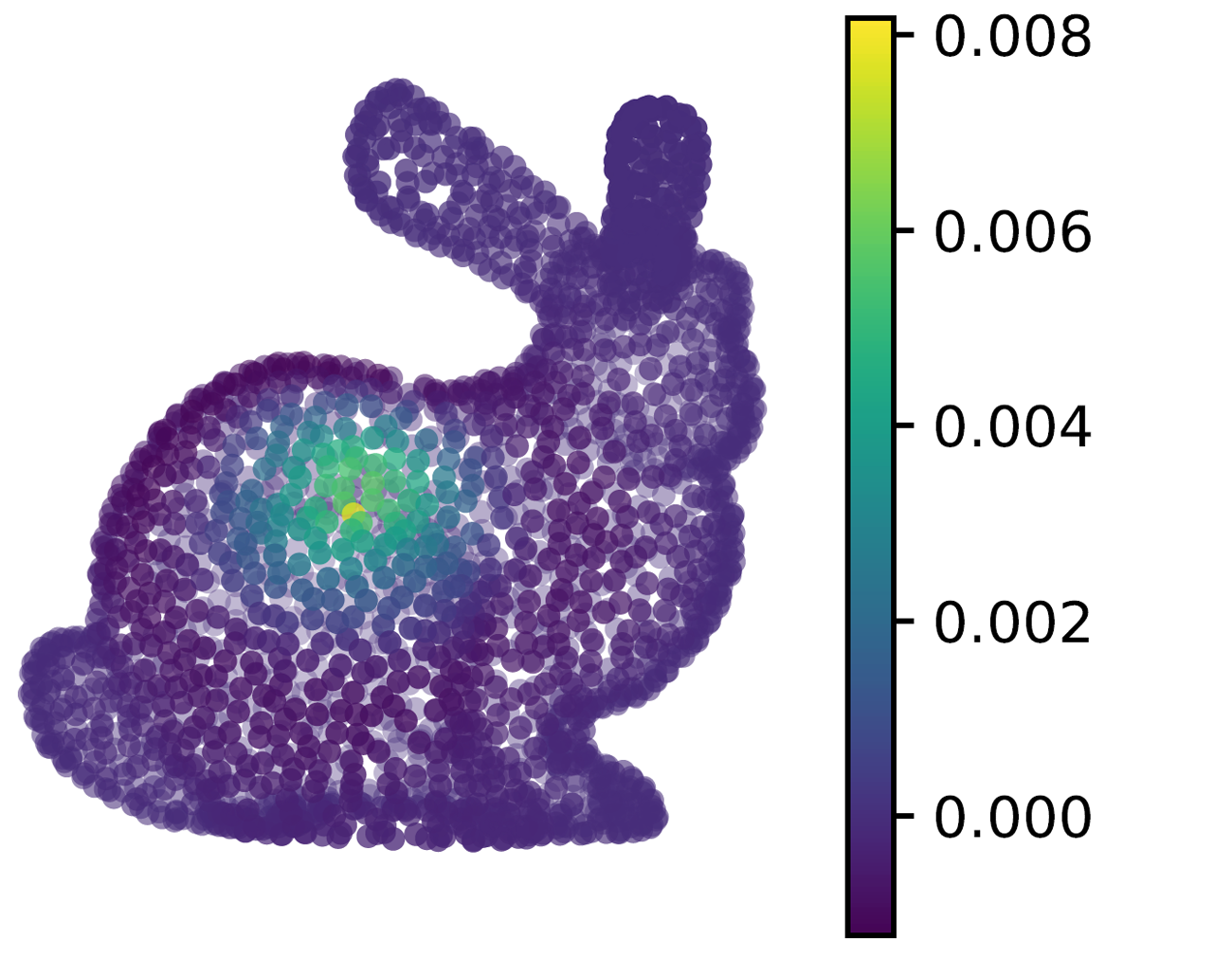}
        \end{minipage}
    }
\caption{The input is the delta function (unit impulse) and the output is the impulse response spreading on point cloud. The color bar indicates the intensity of the signal on point cloud. The input signal is in the spatial domain (on the graph), and for better visualization, we plot the transformed signal on the original graph. (a) Kronecker delta function $\delta_v$ at vertex (node) $v$ (the yellow point with unit intensity); (b) Graph Fourier transform; (c)(d)(e) Mexican hat graph wavelet transform with kernels $h(\lambda)$, $g(s_1\lambda)$ and $g(s_2\lambda)$.}
\label{fig.localization}
\end{figure*}

Powered by graph signal processing~\cite{sandryhaila2013discrete,shuman2013emerging}, it becomes possible to apply frequency analysis for data with underlying irregular structures. Therefore, we utilize this method to address the point cloud by modeling it as a graph to generalize the classical discrete signal processing to the graph domain. Specifically, a point cloud can be represented as a graph by connecting one point with other points in the local neighborhood space. After that, 1) the signal defined on the graph, like point attributes, is firstly transformed into the spectral domain by graph Fourier transform~\cite{sandryhaila2013discrete,shuman2013emerging}, and then 2) the convolution operation is conducted as spectral filtering, i.e., the output signal can be transformed back using the inverse graph Fourier transform.

Graph Fourier transform has been previously introduced for point cloud analysis~\cite{wang2018local,te2018rgcnn,chen2018pointagcn,zhang2018graph,lu2020pointngcnn}, while in this paper we consider employing graph wavelet transform~\cite{hammond2011wavelets}. A crucial property of graph wavelet transform is that wavelet functions are well localized in both the spatial and the spectral domains, as shown in Fig.~\ref{fig.localization}. Specifically, given the input signal, the graph Fourier transform spreads over all frequencies, while the graph wavelet transform localizes on different frequency bands. Previous work like~\cite{zhang2022ms} also uses graph wavelet transform for breast cancer diagnosis, but still fails to take full advantage of multi-scale spectral features. Inspired by~\cite{vaswani2017attention,dosovitskiy2020image}, we introduce a new WaveletFormer layer to explore the relationships between multi-scale spectral embeddings. Specifically, spectral features at different scales are treated as different tokens, and the relationships between them are then captured by the self-attention mechanism.

To apply graph Fourier or wavelet transform, the eigendecomposition of graph Laplacian matrix is usually indispensable. In existing spectral-based graph neural networks~\cite{bruna2013spectral,defferrard2016convolutional,kipf2016semi}, all layers use the full graph as input such that the graph Laplacian matrix construction and the eigendecomposition can be pre-computed. However, the point cloud feature extraction is always conducted in a hierarchical manner, which indicates that the local graph in each layer requires to be dynamically constructed. Though online computation of local graph Laplacian matrix and its eigendecomposition has been proposed in~\cite{wang2018local}, it is still not efficient enough. In contrast, we devise a novel method to bypass the time-consuming task of eigendecomposition. It is done by training an orthogonal matrix, under the assumption that local graph connections can also be learned by neural networks efficiently, which significant accelerates the overall optimization process.

In this paper, we introduce a WaveletFormer layer to process point clouds in the spectral domain via graph wavelet transform~\cite{hammond2011wavelets}, and propose a fast training scheme to avoid high-demanding computation. We then refer to the proposed method as PointWavelet. Extensive experiments on several point cloud datasets demonstrate the effectiveness of the proposed method, which also provides valuable insights on designing spectral graph convolution networks. The main contributions of this paper are as follows:
\begin{enumerate}
\item We introduce a novel method for point cloud analysis in the spectral domain via graph wavelet transform, taking full advantage of different scales of spectral features.
\item We learn to construct the local graph when hierarchically extracting point features and avoid computationally demanding eigendecomposition by a trainable layer.
\item We conduct extensive experiments on various point cloud tasks to validate our method, which inspires the learning-based spectral analysis for point clouds.
\end{enumerate}

\section{Related Work}

\noindent\textbf{Deep Learning on Point Clouds}. Deep learning on point clouds has received increasing attention in the last few years. Generally, a point cloud is not placed on a regular grid, which means each point is independent from others and the number of its neighbors is not fixed. Therefore, applying deep learning to point cloud is always challenging. 
PointNet~\cite{qi2017pointnet} is the pioneering neural network for point set processing, which embeds the input points into high dimensional feature space pointwisely and then uses a symmetric function to aggregate the features for all points in a permutation-invariant manner as the final global feature. The multi-layer perceptrons (MLPs) introduced by this work has been widely used as the building blocks in many other point cloud networks~\cite{qi2017pointnet++,zhao2019pointweb,yan2020pointasnl,wen2021learning} to learn pointwise representations. 
Then, other works such as convolution-based~\cite{li2018pointcnn,wu2019pointconv,thomas2019kpconv}, graph-based~\cite{wang2018local,wang2019dynamic,zhang2018graph}, and volumetric-based~\cite{maturana2015voxnet,riegler2017octnet} methods have been developed to aggregate local hierarchical information and achieved superior performance. 

Inspired by the significant success in NLP and CV, Transformers~\cite{vaswani2017attention} have been applied to point cloud processing~\cite{zhao2021point,guo2021pct,mazur2021cloud,pan20213d,yu2022point,lai2022stratified,hui2021pyramid}. For example, Zhao et al.~\cite{zhao2021point} propose to apply self-attention in the local neighborhood of each point, where the proposed transformer layer is invariant to the permutation of the point set, making it suitable for point set processing tasks. Guo et al.~\cite{guo2021pct} propose a novel point cloud transformer framework or PCT to replace the original self-attention module with a more suitable offset-attention module, which includes an implicit Laplace operator and a normalization refinement.\\

\noindent\textbf{Spectral CNNs}.
In spectral-based methods~\cite{bruna2013spectral, defferrard2016convolutional, kipf2016semi, levie2018cayleynets, susnjara2015accelerated}, graph Fourier transform is used to convert signals from the spatial domain to the spectral domain, which is equivalent to convolution in the spectral domain. This is done by multiplying graph signals and the eigenvectors of the graph Laplacian matrix, which maintains the weight sharing property of CNN.
For example, Bruna et al.~\cite{bruna2013spectral} assume that the filter is a set of learnable parameters and considers graph signals with multiple channels, while it is computationally inefficient due to the eigendecomposition and the filter is non-spatially localized. In follow-up works, Defferrard et al.~\cite{defferrard2016convolutional} propose ChebNet to approximate the filter by Chebyshev polynomials of the diagonal matrix of eigenvalues. As an improvement over~\cite{bruna2013spectral}, the filter defined by ChebNet is localized in space, which means it can extract local features independently of the graph size.  Kipf et al.~\cite{kipf2016semi} further introduce a first-order approximation of ChebNet, namely GCN, simplifying the convolution operation to alleviate the problem of overfitting.
Spectral methods are also well suited for many learning-based applications in point cloud, since point cloud can be efficiently represented by graph with particular connectivity restrictions~\cite{yi2017syncspeccnn,wang2018local}.\\

\noindent\textbf{Graph Wavelet Networks}.
Wavelet transform is an excellent alternative to Fourier transform in signal processing and data compression. In wavelet analysis, a signal is decomposed into components localized in both the spatial and the spectral domains. 
By applying the concept of classical wavelet operators in the spectral domain, \cite{hammond2011wavelets} first define the spectral wavelet transform on arbitrary graphs. In addition, they propose an efficient way to bypass the eigendecomposition of the Laplacian matrix and approximates wavelets with Chebyshev polynomials. Based on the graph wavelet frame, \cite{xu2019graph} introduce a graph diffusion network and
\cite{shen2021multi} propose a multi-scale graph convolutional
network. \cite{wang2020mgcn} use graph wavelet for computing descriptors for characterizing points on three-dimensional surfaces and \cite{zhang2022ms} use graph wavelet for breast cancer diagnosis. It is worth mentioning that \cite{rustamov2013wavelets} introduce a machine learning framework for constructing graph wavelets that can sparsely represent a given class of signals. This framework uses the lifting scheme, and takes into consideration structural properties of both graph signals and their underlying graphs. The lifting scheme is a custom-design construction of biorthogonal wavelets which are different from the classic wavelets, like Haar wavelet and Mexican hat wavelet. Some recent methods also use the lifting-based frame to define graph wavelet networks~\cite{rodriguez2020deep,xu2022graph,huang2021adaptive}.

\section{Preliminaries}
\label{preliminary}
In this section, we introduce the basic concepts of spectral graph theory, i.e., graph Fourier and wavelet transform. 

\subsection{Graph Fourier Transform}\label{graph_fourier_transform}
Given a set $\mathcal{V}$ with $n$ vertices, we then have the weighted graph $\mathcal{G}=(\mathcal{V}, \mathcal{E}, \mathcal{W})$, where $\mathcal{E}$ and $\mathcal{W}$ indicate the graph edges and the weights on graph edges, respectively.
The normalized Laplacian matrix $L \in \mathbb{R}^{n \times n}$ of the graph $\mathcal{G}$ is usually defined as follows:
\begin{equation}
\label{normalized_Laplacian_matrix}
L = I - D^{-\frac{1}{2}} A D^{-\frac{1}{2}},
\end{equation}
where $I$ is an identity matrix of size $n$, $A \in \mathbb{R}^{n \times n}$ indicates the adjacency matrix of $\mathcal{G}$, and $D \in \mathbb{R}^{n\times n}$ is the diagonal degree matrix with the entry $D_{ii} = \sum_j A_{ij}$. Here, given $L$ as a real symmetric positive semi-definite matrix, we  then have its eigendecomposition $L = U \Lambda U^\top$, where each column of $U \in \mathbb{R}^{n \times n}$ indicates an eigenvector of $L$, and $\Lambda$ is the diagonal matrix (which is also known as the spectrum of $L$) whose diagonal elements are the corresponding eigenvalues, i.e., $\Lambda_{ii} = \lambda_i$. 
Therefore, given a signal $f \in \mathbb{R}^{n}$, we have its graph Fourier transform as $\hat{f} = U^\top f$ and its inverse graph Fourier transform as $f = U \hat{f}$.
The definition of convolution operator on a graph can be enabled by the graph Fourier transform~\cite{defferrard2016convolutional}, i.e., the graph convolution ($*_\mathcal{G}$) with a graph filter $g \in \mathbb{R}^n$ can be defined as
\begin{equation}
\label{graph_convolution}
g *_\mathcal{G} f = U((U^\top g) \odot (U^\top f)),
\end{equation}
where $\odot$ indicates the element-wise Hadamard product operation. In addition, Eq.~\eqref{graph_convolution} can be further simplified as 
\begin{equation}
g *_\mathcal{G} f = Ug_{\theta}U^\top f,
\end{equation}
where $g_{\theta} = \mathop{\mathrm{diag}}(U^\top g)$ usually indicates a graph filter with a set of trainable parameters in spectral CNNs ~\cite{bruna2013spectral,defferrard2016convolutional,kipf2016semi,wu2020comprehensive}.

\subsection{Graph Wavelet Transform}
\label{graph_wavelet_transform}

The (spectral) graph wavelet transform is first introduced in~\cite{hammond2011wavelets} based on the graph Fourier transform. Given a band-pass filter $g(\cdot)$ (also named wavelet function) with the parameter $s \in \mathbb{R}^+$ called scale, we can define a diagonal matrix $g_s = \mathop{\mathrm{diag}}(g(s\lambda_1), \dots, g(s\lambda_i), \dots, g(s\lambda_n))$, where $\lambda_i$ is the $i$-th eigenvalue of $L$. Then, the wavelet basis at the scale $s$ can be defined as
\begin{equation}
\label{equation_wavelets_s}
\Psi_s = [\psi_{s, 1}, \dots, \psi_{s, n}] = Ug_sU^\top.
\end{equation}
For the wavelet centered at a vertex $v \in \mathcal{V}$, we have 
\begin{equation}
\psi_{s, v} = Ug_sU^\top \delta_v,
\end{equation}
where $\delta_v$ is a one-hot vector for the vertex $v$ (the delta function in Fig.~\ref{fig.localization}). 
Note that, though the filter $g(\cdot)$ is a continuous function defined on $\mathbb{R}^+$, $g_s$ is discrete with different scale $s$, i.e., only $g(s\lambda_i), i=1, 2, ..., n$ are required. Another important concept  is the scaling function $h(\cdot)$ which is a low-pass filter that captures low-frequency information. Note that, $h(\cdot)$ does not use the parameter $s$. Similarly, we can define $\Psi_0$ via $\mathop{\mathrm{diag}}(h(\lambda_1), \dots, h(\lambda_i), \dots, h(\lambda_n))$.

Like the graph Fourier transform, graph wavelet transform projects graph signals from the spatial domain into the spectral domain. For example, by using wavelet basis at scale $s$, the graph wavelet transform of signal $f$ can then be defined as $\hat{f}_s = \Psi_{s} f$ and consequently the inverse graph wavelet transform is $f = \Psi_{s}^{-1} \hat{f}_s$. Generally, different scales always imply different receptive fields in the spectral domain. On the one hand, with a small scale $s$, the filter $g(\cdot)$ is stretched out and it lets through high-frequency modes essential to good localization, which means the corresponding wavelets extend only to their close neighborhoods on the graph; On the other hand, with a large scale $s$, the filter $g(\cdot)$ is compressed around low-frequency modes. This encourages wavelets encoding a coarser description of the local environment.

\subsection{Spectral Graph Convolution}
\label{spectral_graph_convolution}

When applying graph Fourier or wavelet transform to the network, the $i$-th spectral graph convolutional layer can be defined as~\cite{wu2020comprehensive,zhang2020deep} 
\begin{equation}
f^{(out)} = \sigma(ITran(\Theta^{(i)} Tran(f^{(in)}))),
\end{equation}
where $f^{(in)} \in \mathbb{R}^{n}$ and $f^{(out)} \in \mathbb{R}^{n}$ denote the input and output features, respectively. $Tran(\cdot)$ is the graph Fourier or wavelet transform, $ITran(\cdot)$ is the corresponding inverse transform, $\Theta^{(i)}$ is the trainable filter learned in the spectral domain, and $\sigma(\cdot)$ denotes the non-linear activation function.

\section{Methodology}
\label{methodology}

In this section, we introduce the proposed WaveletFormer layer via graph wavelet transform and the overall PointWavelet framework. In addition, we also devise an efficient learnable strategy for graph wavelet transform.

\subsection{WaveletFormer Layer}
\label{waveletformer}

To explore multi-scale convolutions in the spectral domain, we introduce wavelets to transform graph signals into the spectral domain at different scales. To achieve this, we devise our spectral neural networks by utilizing the graph wavelet transform as follows. To construct the wavelet $\psi_{s, v}$, we apply a filter $g(\cdot)$, which is usually defined by Haar function, Morlet function, or Mexican hat function. In this paper, we use the Mexican hat function. In addition to the wavelet function, we use the same type of function with different parameters for the single scaling function $h(\cdot)$ that captures low-frequency information. Generally, $h(\cdot)$ and $g(\cdot)$ cannot be chosen arbitrarily. To form a stable recovery of $f$ from its scaling and wavelet function coefficients, $h^2(\lambda) + \sum_{s} g^2(s\lambda)$ should be nonzero for all $\lambda$ on the spectrum of $L$~\cite{hammond2011wavelets}. An example of $h(\cdot)$ and $g(\cdot)$ at different scales is shown in Fig.~\ref{fig.scaling_and_wavelet}.

\begin{figure}[!ht]
    \centering
    \includegraphics[width=\linewidth]{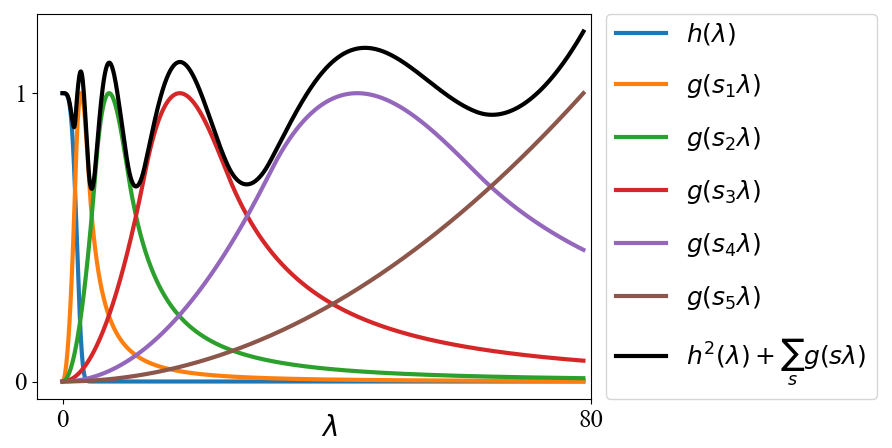}
    \caption{An example of $h(\lambda)$ (the blue curve) and $g(s_i\lambda)$. The sum $h^2(\lambda) + \sum_{s} g^2(s\lambda)$ (the dark curve) is nonzero over all frequencies.}
    \label{fig.scaling_and_wavelet}
\end{figure}

Inspired by Transformer architectures~\cite{vaswani2017attention,dosovitskiy2020image}, we explore the relationships between multi-scale spectral representations and introduce a new WaveletFormer layer as follows. 
Given a set of scales $\{s_1, s_2, \dots, s_J\}$, where $J$ indicates the number of scales, the graph wavelet transform is then defined as a linear mapping from $\mathbb{R}^{n}$ to the corresponding scaling and wavelet function coefficient domain $\mathbb{R}^{(1+J) \times n}$. Specifically, given a point embedding $f \in \mathbb{R}^{n}$, we first transform it into the spectral domain via graph wavelet transform, i.e., we have the spectral embeddings $[\hat{f}_0, \hat{f}_{s_1}, \dots, \hat{f}_{s_J}] = [\Psi_{0} f, \Psi_{s_1} f, \dots, \Psi_{s_J} f]$, which contains the spectral information at different scales. After that, these spectral embeddings are feed into the typical transformer encoder layer~\cite{dosovitskiy2020image} to obtain the output spectral features $[\hat{f}^{'}_0, \hat{f}^{'}_{s_1}, \dots, \hat{f}^{'}_{s_J}]$. We concatenate the spectral embeddings at different scales as the final multi-scale representation. An overview of the proposed WaveletFormer layer is depicted in Fig.~\ref{fig.waveletformer}.

\begin{figure}[!ht]
    \centering
    \includegraphics[width=\linewidth]{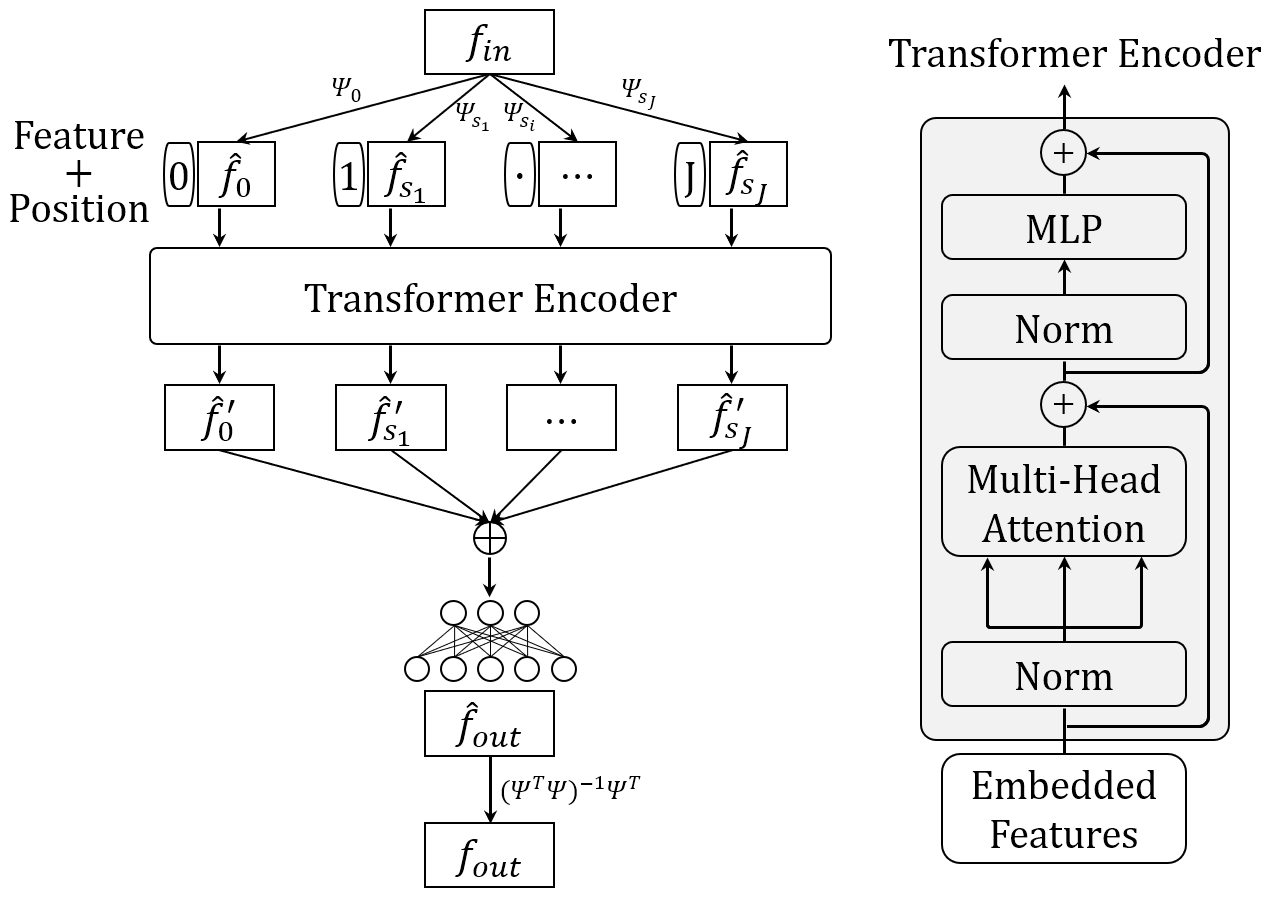}
    \caption{An overview of the proposed WaveletFormer layer. Specifically, we first construct a multi-scale spectral input via graph wavelet transform, which is then fed into a transformer encoder layer~\cite{dosovitskiy2020image}. Lastly, the output spectral embeddings are concatenated together to form the final multi-scale spectral representation.
    }
    \label{fig.waveletformer}
\end{figure}

In frequency analysis, it is of great importance that the graph wavelet transform is invertible, i.e., we can reconstruct the signal from a given set of scaling and wavelet function coefficients. As mentioned before, the graph wavelet transform computes $(1+J) \times n$ scaling and wavelet function coefficients for each $n$ dimensional input $f$. However, as the number of coefficients is larger than the dimension of the original input $f$, the graph wavelet transform is an overcomplete transform and thus cannot have a unique linear inverse. Considering there will be infinitely many different left inverse matrices $\Psi^{-1}$ satisfying $\Psi^{-1}\Psi=I$, we use the pseudoinverse $(\Psi^\top\Psi)^{-1}\Psi^\top$ as the inverse of the graph wavelet transform~\cite{hammond2011wavelets}. Note that, $g_s$ is a diagonal matrix from Eq.~\eqref{equation_wavelets_s}, we thus have $\Psi_s = \Psi_s^\top$. That is, it is easy to evaluate $(\Psi^\top \Psi)^{-1}\Psi^\top$. Specifically, given $p(x) = h^2(x) + g^2(s_1x) + \cdots + g^2(s_Jx)$, we then have 
\begin{equation}
(\Psi^\top \Psi)^{-1} = U p_{\lambda}^{-1} U^\top,
\end{equation}
where $p_{\lambda}^{-1} = \mathop{\mathrm{diag}}(p^{-1}(\lambda_1), p^{-1}(\lambda_2), \cdots, p^{-1}(\lambda_n)) \in \mathbb{R}^{n \times n}$. We would like to refer interested readers to supplementary materials for more details about this calculation.

\begin{figure*}[!t]
    \centering
    \includegraphics[width=\linewidth]{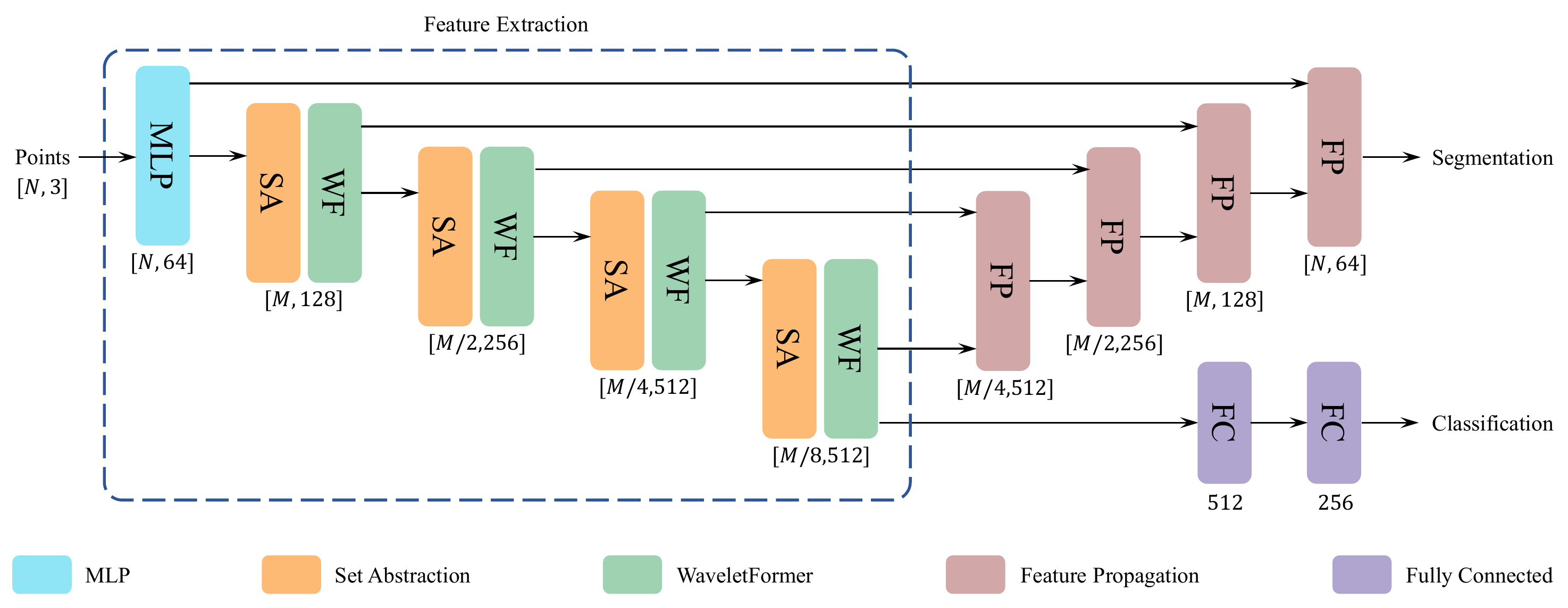}
    \caption{The main PointWavelet architecture. We use the set abstraction (SA) layer~\cite{qi2017pointnet++} for points sampling and grouping. Here, $N$ is the number of input points and $M$ is the grouping centers of the first SA layer. After each SA layer, we attach a WaveletFormer (WF) layer for spectral feature extraction. For point cloud segmentation, we also reuse the feature propagation (FP) layer as in \cite{qi2017pointnet++}.
    }
    \label{fig.main}
\end{figure*}

\subsection{PointWavelet Architecture}
We illustrate the overall PointWavelet architecture for point cloud spectral analysis in Fig.~\ref{fig.main}. For feature extraction, we use the set abstraction (SA) layer~\cite{qi2017pointnet++} for grouping and sampling, and after that we attach an extra WaveletFormer layer for spectral embedding. 
In our pipeline, we use a stack of three the above-mentioned blocks as the backbone network. In addition, we also reuse the feature propagation (FP) layer for segmentation~\cite{qi2017pointnet++} and more details can be found in the experiments section.
A difference between spatial- and spectral-based methods is that the latter always transforms the input into the spectral domain before feature extraction. Contrastingly, the main difference between the proposed PointWavelet and existing spectral-based methods is that PointWavelet extracts spectral features using multi-scale wavelet transform via a new WaveletFormer layer.

\subsection{Learnable Wavelet Transform}
As defined in Eq.~\eqref{equation_wavelets_s}, applying  graph wavelet transform requires the eigendecomposition of $L$ with computational complexity of $O(n^3)$. Therefore, this time-consuming operation poses a great challenge to many real-world applications. Similar problem also arises when applying spectral filtering using graph Fourier transform~\cite{wang2018local,te2018rgcnn,chen2018pointagcn,lu2020pointngcnn,zhang2018graph}. To alleviate this time-consuming operation, previous works such as~\cite{hammond2011wavelets,defferrard2016convolutional,kipf2016semi} usually assume the filter $g_{\theta}$ to be a polynomial of the eigenvalue of $L$.
As a result, the spectral filtering can be approximated by Chebyshev polynomials with degraded performance to avoid time-consuming eigendecomposition. Though the evaluation of graph wavelet transform can also adopt this idea, here we introduce a novel and efficient strategy by leveraging a learnable wavelet transform.

\begin{theorem}
\label{Theorem}
Given a set of vertices $\mathcal{V}$ with $|\mathcal{V}| = n$, there exists one orthogonal matrix $U = [\bm{u_1}, \bm{u_2}, ..., \bm{u_n}] \in \mathbb{R}^{n \times n}$ for which $\bm{u_1}=(c, c, ..., c)^\top \in \mathbb{R}^{n}$ with $c \neq 0$, and a set of $\lambda_i$ satisfying $0 = \lambda_1 \leq \lambda_2 ... \leq \lambda_n$, such that $U \Lambda U^\top \in \mathbb{R}^{n \times n}$ forms the Laplacian matrix of local graph $\mathcal{G}$ , where $\Lambda = \mathop{\mathrm{diag}}(\lambda_1, \lambda_2, ..., \lambda_n)$.
\end{theorem}
\begin{proof}
In supplementary.
\end{proof}

From Theorem~\ref{Theorem}, we have that the connections of graph vertices can be learned by neural networks. 
Inspired by this, a naive way is to initialize a matrix $U \in \mathbb{R}^{n \times n}$ with $\bm{u_1}=(c, c, ..., c)^\top$ and then  enforce $UU^\top=I$ during the training. That is, if the matrix Frobenius norm is small enough, $\|I - UU^\top \|_F^2 < \epsilon$, we then consider $U$ to be an orthogonal matrix. However, the complexity of the above-mentioned method is $O(n^2)$. Instead, we propose a more efficient strategy to generate a full set of normalized vectors that are orthogonal to each other. Given a normalized vector $\bm{q}=(q_1, q_2, \dots, q_n)^\top \in \mathbb{R}^{n}$, i.e., $\|\bm{q}\|_2=1$, we then define
the orthogonal matrix $U$ as

\begin{equation}
\label{U_construction_2}
U = 
\left[
\begin{array}{cccc}
q_{1}   & -q_{2}     & \cdots & -q_{n}    \\
q_{2}   & F(2, 2) + 1  & \cdots & F(2, n) \\
\vdots  & \vdots     & \ddots & \vdots    \\
q_{n}   & F(n, 2)  & \cdots & F(n, n)+1
\end{array}
\right],
\end{equation}
where
\begin{equation}\label{U_construction_1}
F(i, j) = \frac{q_i q_j}{\sum_{k=2}^{n} |q_{k}|^2} (q_{1} -1), \, 2 \leq i, j \leq n.
\end{equation}

As the first column of $U$ is identical to vector $\bm{q}$, it is also easy to verify that 
$UU^\top=I$ (see more details in our supplementary). 
With Eq.~\eqref{U_construction_2}, we can construct the orthogonal matrix $U$ using a vector $\bm{q}$ with only $n$ free variables, which facilitates the efficient learning in the spectral domain.
Therefore, different from existing methods, our method directly learns a matrix $U$ to represent the local graph connection, instead of evaluating a pairwise distance matrix. 
This strategy also provides a very insightful perspective for point cloud applications. Specifically, in a hierarchical feature representation learning framework, each layer extracts local features by aggregating neighborhood points that make up a local graph. When going deeper, the local graph is updated dynamically with a larger reception field. 
In other words, the local graph connection in each layer is learned dynamically instead of being a fixed input. This is a critical distinction between our method and graph CNNs working on a fixed input graph. 
Therefore, it substantially simplifies and speeds up the network training process. 

In addition, to guarantee that $\bm{q}$ satisfies the required condition, we use two trainable vectors, i.e., $\bm{q} = \bm{q_{ini}} + \bm{q_{eps}}$, where $\bm{q_{ini}} = (c, c, ..., c)^\top \in \mathbb{R}^{n}$ with $c \neq 0$ and $\bm{q_{eps}} = (e_{1}, e_{2}, ..., e_{n})^\top \in \mathbb{R}^{n}$. During training, we then force $\|\bm{q_{eps}}\|_1 \rightarrow 0$. 
Given $\mathcal{L}_{task}$ as the task-specific loss function for PointWavelet, we then have the loss function for PointWavelet with learnable wavelet transform, i.e., PointWavelet-L, as follows:
\begin{equation}
\mathcal{L} = \mathcal{L}_{task} + \beta \sum_{i}\|\bm{q_{eps, i}}\|_1, 
\end{equation}
where $i$ indicates the index of each layer with a learnable orthogonal matrix and $\beta \geq 0$ is the weight of the second loss term.

\section{Experiments}

In this section, we first introduce the datasets used in our experiments and the implementation details in training/testing. We then compare the proposed method with recent state-of-the-art methods on different point cloud tasks such as classification and segmentation. Lastly, we perform several ablation studies on graph wavelet transform to better understand the proposed method. If not otherwise stated, we refer to \textbf{PointWavelet} as graph wavelet transform by eigendecomposition, and \textbf{PointWavelet-L} as graph wavelet transform by the proposed learning strategy.

\subsection{Datasets}
We evaluate the proposed method on point cloud classification and segmentation using the following four popular point cloud datasets. 
\begin{enumerate}[]
\item ModelNet40~\cite{wu20153d}: It contains CAD models from 40 categories. We use the official split of 9,843 training and 2,468 testing samples.
\vspace{1mm}
\item ScanObjectNN~\cite{uy2019revisiting}: It contains about 15,000 real scanned objects that are categorized into 15 classes with 2,902 unique object instances.
\vspace{1mm}
\item ShapeNet-Part~\cite{yi2016scalable}: It contains 16,881 models from 16 shape categories, with total 50 different parts labeled. We use the official split, i.e., 14,006 for training and 2,874 for testing.
\vspace{1mm}
\item S3DIS~\cite{armeni20163d}: It contains 3D scans in six large-scale indoor areas including 271 rooms. Each point is annotated with a label from 13 semantic classes, such as chair and sofa.
\end{enumerate}

\subsection{Implementation Details}
\label{implementation_detail}

We implement the proposed method using PyTorch~\cite{paszke2019pytorch}. To train our network, we first randomly initialize the trainable vector $\lambda_{\theta_2}, \dots, \lambda_{\theta_n}$ and then have $\lambda_{i} = (tanh(\lambda_{\theta_i}) + 1)$ to ensure $\lambda_{i}> 0$. 
Similarly, we initialize $\bm{q_{ini}} = (c_{\theta}, c_{\theta}, ..., c_{\theta})^\top \in \mathbb{R}^{n}$ with $c_{\theta} \neq 0$ and $\bm{q_{eps}} = (e_{\theta_1}, e_{\theta_2}, ..., e_{\theta_n})^\top \in \mathbb{R}^{n}$.
In our experiments, if not otherwise stated, we use $\beta=0.05$, $J=5$, and the Mexican hat wavelet with the scaling function is $h(x) = e^{-x^4}$ and the wavelet function is $g(x) = xe^{-x}$. 
Since the proposed method mainly uses the set abstraction layer~\cite{qi2017pointnet++} and the transformer encoder layer~\cite{dosovitskiy2020image}, we show their detailed configurations in Table~\ref{tbl.structure_sa} and Table~\ref{tbl.structure_transformer}, respectively. For each SA layer in Table~\ref{tbl.structure_sa}, the first row is the number of centroids, the second row is the size of local neighbors and the last row is the output feature channel.

\begin{table}[h]
    \caption{The configurations of set abstraction (SA) layers.}
    \centering
    \resizebox{0.6\columnwidth}{!}{
    \begin{tabular}{ccccc}
    \toprule
     &\multirow{1}{*}{SA1} &\multirow{1}{*}{SA2}  &\multirow{1}{*}{SA3} &\multirow{1}{*}{SA4}\\
    \midrule
    \multirow{1}*{Centroids}
    &512  &128  &32   &1    \\
    \multirow{1}*{$k$-NN}
    &32   &32   &32   &32   \\
    \multirow{1}*{Channels}
    &128  &256  &512  &512   \\
    \bottomrule
    \end{tabular}}
    \label{tbl.structure_sa}
\end{table}

\begin{table}[h]
    \caption{The configurations of WaveletFormer (WF) layers.}
    \centering
    \resizebox{0.6\columnwidth}{!}{
    \begin{tabular}{ccccc}
    \toprule
     &\multirow{1}{*}{WF1} &\multirow{1}{*}{WF2}  &\multirow{1}{*}{WF3} &\multirow{1}{*}{WF4}\\
    \midrule
    \multirow{1}*{Encoders}
    &2   &2    &2    &2   \\
    \multirow{1}*{Dimension}
    &128 &256  &512  &512   \\
    \multirow{1}*{Heads}
    &4   &4    &4    &4    \\
    \bottomrule
    \end{tabular}}
    \label{tbl.structure_transformer}
\end{table}

\subsection{Point Cloud Classification}
We first evaluate our method on ModelNet40 dataset for point cloud classification.
For fair comparison, we use 1024 points sampled from mesh objects in ModelNet40 without normals as input, and the evaluation metrics are overall accuracy (OA) and average class accuracy (mAcc). We compare with both spatial-based methods like PointNet++~\cite{qi2017pointnet++} and spectral-based methods like LSGCN~\cite{wang2018local}. Experimental results presented in Table~\ref{tbl.classification.MN40} demonstrate that our PointWavelet and PointWavelet-L clearly outperform the state-of-the-art methods. Generally, PointWavelet-L achieves comparable even better performance than PointWavelet because PointWavelet-L always learns to construct the orthogonal matrix $U$ which implies the optimal local graph connections rather than that pairwise defined in PointWavelet.

\begin{table}[h]
    \caption{Point cloud classification results on ModelNet40~\cite{wu20153d}.}
    \centering
    \resizebox{\columnwidth}{!}{
    \renewcommand\arraystretch{1.1}
    \begin{tabular}{llcccc}
    \toprule
    \multicolumn{2}{c}{Model} &\multirow{1}{*}{Input} &\multirow{1}{*}{OA(\%)} &\multirow{1}{*}{mAcc(\%)}  \\
    \midrule
    \multirow{13}*{\rotatebox{-90}{Spatial-based}}
    &PointNet~\cite{qi2017pointnet} &1k P  &89.2  &86.2  \\
    \multirow{1}*{}
    &PointNet++~\cite{qi2017pointnet++} &1k P  &90.7  &-  \\
    \multirow{1}*{}
    &PointCNN~\cite{li2018pointcnn} &1k P  &92.2  &88.1  \\
    \multirow{1}*{}
    &PointConv~\cite{wu2019pointconv} &1k P+N  &92.5  &-  \\
    \multirow{1}*{}
    &KPConv~\cite{thomas2019kpconv} &7k P  &92.9  &-  \\
    \multirow{1}*{}
    &DGCNN~\cite{wang2019dynamic} &1k P  &92.9  &90.2  \\
    \multirow{1}*{}
    &DensePoint~\cite{liu2019densepoint} &1k P  &93.2  \\
    \multirow{1}*{}
    &RS-CNN~\cite{liu2019relation} &1k P  &93.6  &-  \\
    \multirow{1}*{}
    &Point Trans~\cite{zhao2021point} &1k P  &93.7  &90.6  \\
    \multirow{1}*{}
    &PCT~\cite{guo2021pct} &1k P  &93.2  &-  \\
    \multirow{1}*{}
    &CurveNet~\cite{xiang2021walk} &1k P  &93.8  &-  \\
    \multirow{1}*{}
    &PointNeXt~\cite{qian2022pointnext} &1k P &93.2 &90.8 \\
    \multirow{1}*{}
    &PointMixer~\cite{choe2022pointmixer} &1k P &93.6 &\textbf{91.4} \\
    \multirow{1}*{}
    &PointMLP~\cite{ma2022rethinking} &1k P  &94.1  &90.9  \\
    \midrule
    \multirow{7}*{\rotatebox{-90}{Spectral-based}}
    &LSGCN~\cite{wang2018local} &1k P+N  &91.8  &-  \\
    \multirow{1}*{}
    &PointAGCN~\cite{chen2018pointagcn} &1k P  &91.4  &-  \\
    \multirow{1}*{}
    &PointGCN~\cite{zhang2018graph} &1k P  &89.5  &86.1  \\
    \multirow{1}*{}
    &RGCNN~\cite{te2018rgcnn} &1k P+N  &90.5  &87.3  \\
    \multirow{1}*{}
    &PointNGCNN~\cite{lu2020pointngcnn} &1k P  &92.8  &-  \\
    \multirow{1}*{}
    &AWT-Net~\cite{huang2021adaptive} &1k P  &93.9  &-  \\
    \multirow{1}*{}
    &PointWavelet &1k P  &94.1  &91.1  \\
    \multirow{1}*{}
    &PointWavelet-L &1k P  &\textbf{94.3}  &91.1  \\
    \bottomrule
    \end{tabular}}
    \label{tbl.classification.MN40}
\end{table}

The objects in ModelNet40 are synthetic, and thus are complete, well-segmented, and noise-free. However, objects obtained from real-world 3D scans are significantly different from CAD models due to the presence of background noise and the non-uniform density due to holes from incomplete scans/reconstructions and occlusions. To further validate our method, we perform experiemnt on a real-world dataset ScanObjectNN and follow PointMLP~\cite{ma2022rethinking} to use 1024 points as input without normals.
As shown in Table~\ref{tbl.classification.ScanObjectNN}, our method also achieve competitive performance with recent state-of-art methods.

\begin{table}[!ht]
    \caption{Point cloud classification results on ScanObjectNN~\cite{uy2019revisiting}.}
    \centering
    \resizebox{0.7\columnwidth}{!}{
    \renewcommand\arraystretch{1.1}
    \begin{tabular}{lccc}
    \toprule
    \multicolumn{1}{c}{Model} &\multirow{1}{*}{OA(\%)} &\multirow{1}{*}{mAcc(\%)}  \\
    \midrule
    \multirow{1}*{}
    PointNet~\cite{qi2017pointnet} &68.2  &63.4  \\
    \multirow{1}*{}
    SpiderCNN~\cite{xu2018spidercnn} &73.7 &69.8  \\
    \multirow{1}*{}
    PointNet++~\cite{qi2017pointnet++} &77.9  &75.4  \\
    \multirow{1}*{}
    PointCNN~\cite{li2018pointcnn} &78.5  &75.1  \\
    \multirow{1}*{}
    DGCNN~\cite{wang2019dynamic} &78.1  &73.6  \\
    \multirow{1}*{}
    SimpleView~\cite{goyal2021revisiting} &80.1  &-  \\
    \multirow{1}*{}
    PRANet~\cite{cheng2021net} &82.1  &79.1  \\
    \multirow{1}*{}
    PointMLP~\cite{ma2022rethinking} &85.4  &83.9  \\
    \multirow{1}*{}
    PointNeXt~\cite{qian2022pointnext} &87.7 &85.8 \\
    \multirow{1}*{}
    PointWavelet &\textbf{87.9}  &85.5  \\
    \multirow{1}*{}
    PointWavelet-L &87.7  &\textbf{85.8}  \\
    \bottomrule
    \end{tabular}}
    \label{tbl.classification.ScanObjectNN}
\end{table}

\subsection{Point Cloud Part Segmentation}
Point cloud segmentation is to assign a label to each point in the set, which is more challenging than classification. We evaluate the proposed method for the segmentation on the ShapeNet-Part dataset, and the experimental setup is the same as~\cite{qi2017pointnet++}. 
We use the input with 2048 points and the evaluation metrics are average instance intersection-over-union(Inst. mIoU) and average class intersection-over-union (Cls. mIoU). The experimental results are shown in Table~\ref{tbl.segmentation}, where the proposed method achieves better performances than both recent state-of-the-art spatial and spectral methods. In addition, we show the part segmentation results of two representative methods, one spatial-based (PointNet++) and one spectral-based (LSGCN), in Fig~\ref{fig.segmentation}. Specifically, compared to PointNet++ and LSGCN, our approach gives results that are closer to the ground truth and better captures fine local structures, such as the empennage of the airplane and the feet of the chair. Generally, graph wavelet transform provides a more faithful representation of changes in local geometry, which allows us to successfully segment connected parts of 3D objects.

\begin{table}[!ht]
    \caption{Point cloud part segmentation on ShapeNet-Part~\cite{yi2016scalable}.}
    \centering
    \resizebox{\columnwidth}{!}{
    \centering
    \renewcommand\arraystretch{1.1}
    \begin{tabular}{llcccc}
    \toprule
    \multicolumn{2}{c}{Model} &\multirow{1}{*}{Inst. mIoU(\%)} &\multirow{1}{*}{Cls. mIoU(\%)} \\
    \midrule
    \multirow{13}*{\rotatebox{-90}{Spatial-based}}
    &PointNet~\cite{qi2017pointnet} &83.7  &80.4  \\
    \multirow{1}*{}
    &PointNet++~\cite{qi2017pointnet++} &85.1  &81.9  \\
    \multirow{1}*{}
    &PointCNN~\cite{li2018pointcnn} &86.1  &84.6  \\
    \multirow{1}*{}
    &PointConv~\cite{wu2019pointconv} &85.7  &82.8  \\
    \multirow{1}*{}
    &KPConv~\cite{thomas2019kpconv} &86.4  &85.1  \\
    \multirow{1}*{}
    &DGCNN~\cite{wang2019dynamic} &85.2  &82.3  \\
    \multirow{1}*{}
    &DensePoint~\cite{liu2019densepoint} &86.4  &84.2  \\
    \multirow{1}*{}
    &RS-CNN~\cite{liu2019relation} &86.2  &84.0  \\
    \multirow{1}*{}
    &Kd-Net~\cite{klokov2017escape} &82.3  &-  \\
    \multirow{1}*{}
    &SO-Net~\cite{li2018so} &84.9  &-  \\
    \multirow{1}*{}
    &Point Trans~\cite{zhao2021point} &86.6  &83.7  \\
    \multirow{1}*{}
    &PCT~\cite{guo2021pct} &86.4  &-  \\
    \multirow{1}*{}
    &CurveNet~\cite{xiang2021walk} &85.9  &-  \\
    \multirow{1}*{}
    &PointMLP~\cite{ma2022rethinking} &86.1  &84.6  \\
    \multirow{1}*{}
    &PointNeXt~\cite{qian2022pointnext} &\textbf{87.0}  &85.2  \\
    \midrule
    \multirow{7}*{\rotatebox{-90}{Spectral-based}}
    &SyncSpec~\cite{yi2017syncspeccnn} &84.7  &82.0  \\
    \multirow{1}*{}
    &LSGCN~\cite{wang2018local} &85.4  &-  \\
    \multirow{1}*{}
    &PointAGCN~\cite{chen2018pointagcn} &85.6  &-  \\
    \multirow{1}*{}
    &RGCNN~\cite{te2018rgcnn} &84.3  &-  \\
    \multirow{1}*{}
    &PointNGCNN~\cite{lu2020pointngcnn} &85.6  &82.4  \\
    \multirow{1}*{}
    &AWT-Net~\cite{huang2021adaptive} &86.6  &85.0  \\
    \multirow{1}*{}
    &PointWavelet &86.8  &85.0  \\
    \multirow{1}*{}
    &PointWavelet-L &86.8  &\textbf{85.2}  \\
    \bottomrule
    \end{tabular}}
    \label{tbl.segmentation}
\end{table}

\begin{figure}[!ht]
    \subfloat[PN++]{
        \begin{minipage}[b]{0.23\linewidth} 
        \includegraphics[width=1.0\columnwidth]{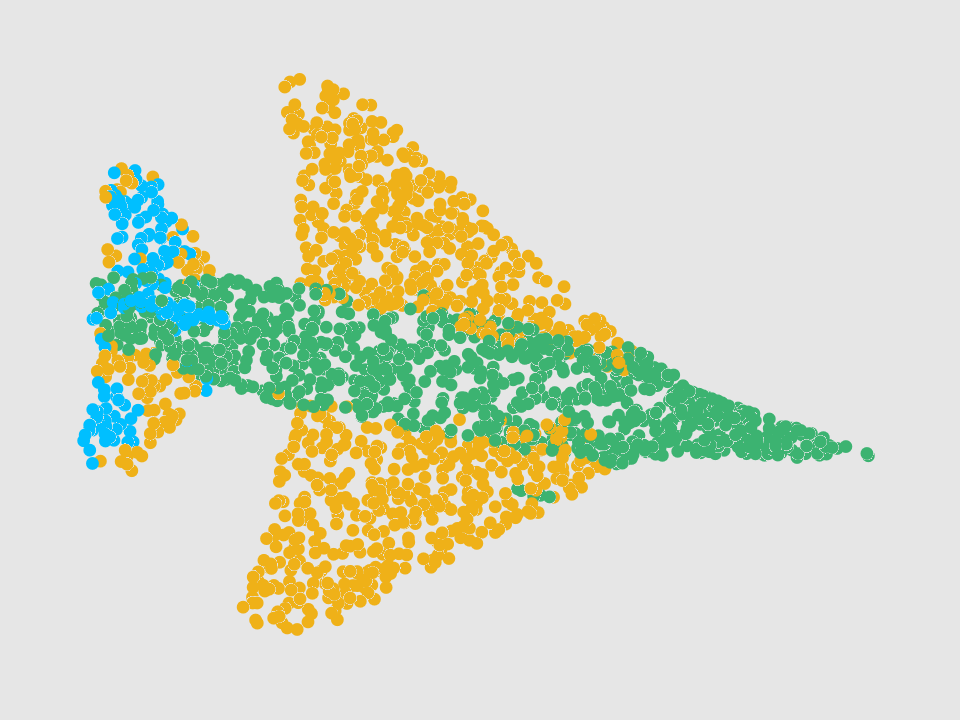} \\
        \vspace{-3mm}
        \includegraphics[width=1.0\columnwidth]{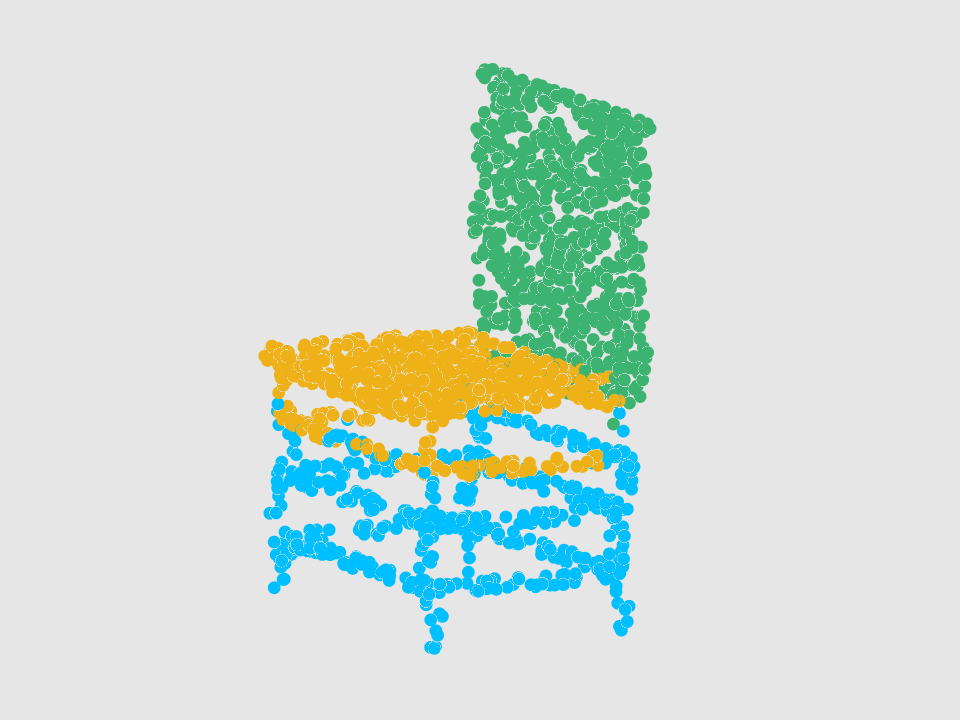} \\
        \vspace{-3mm}
        \includegraphics[width=1.0\columnwidth]{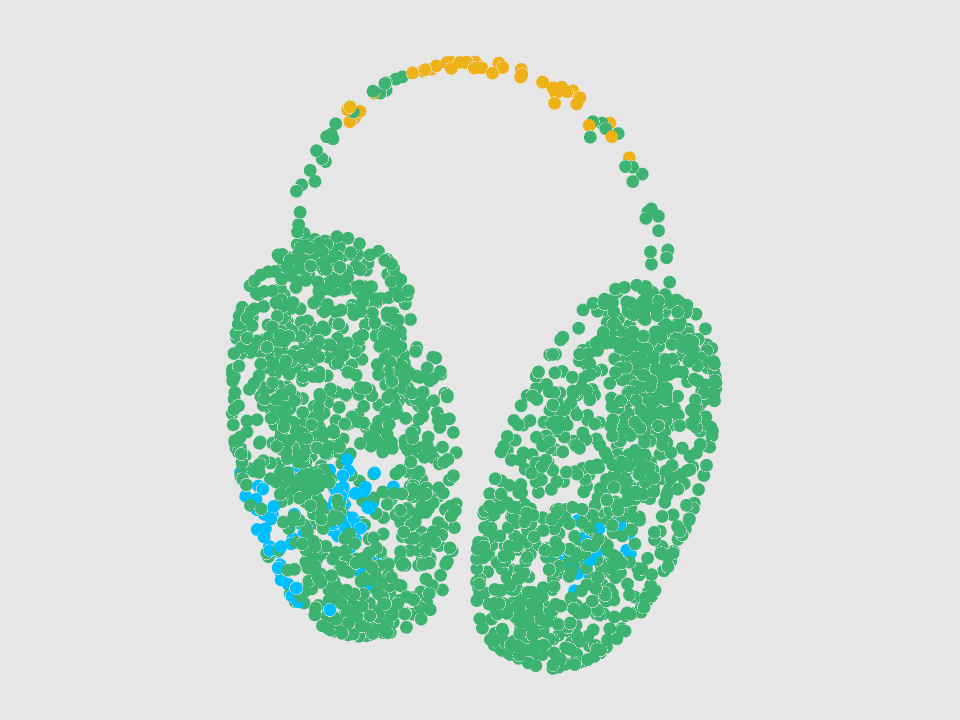}
        \end{minipage}
    }
    \subfloat[LSGCN]{
        \begin{minipage}[b]{0.23\linewidth} 
        \includegraphics[width=1.0\columnwidth]{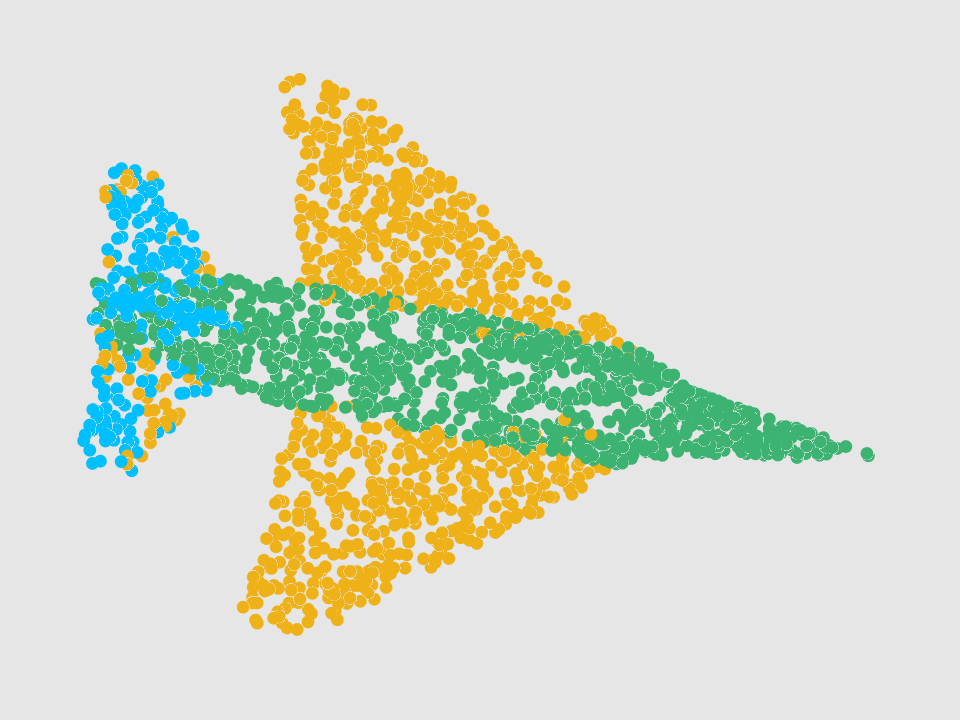} \\
        \vspace{-3mm}
        \includegraphics[width=1.0\columnwidth]{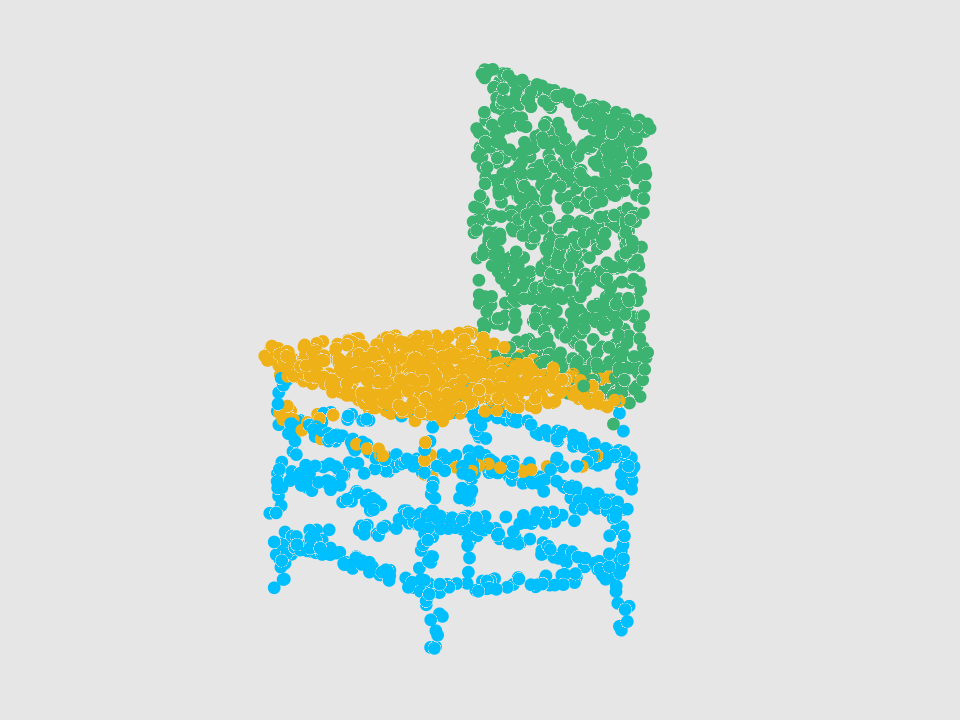} \\
        \vspace{-3mm}
        \includegraphics[width=1.0\columnwidth]{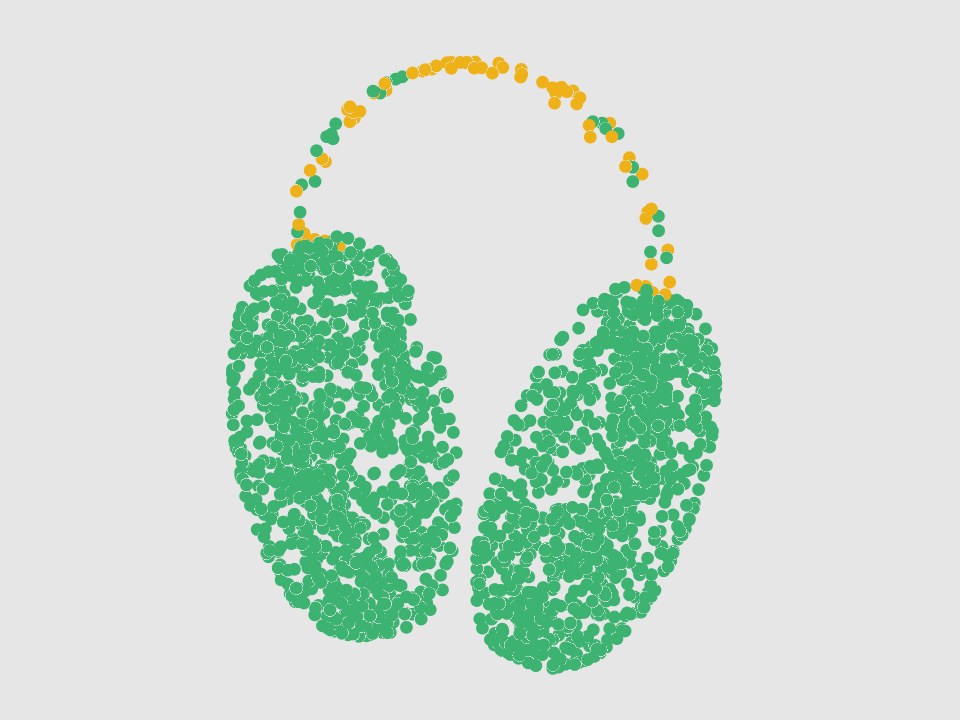}
        \end{minipage}
    }
    \subfloat[Ours]{
        \begin{minipage}[b]{0.23\linewidth} 
        \includegraphics[width=1.0\columnwidth]{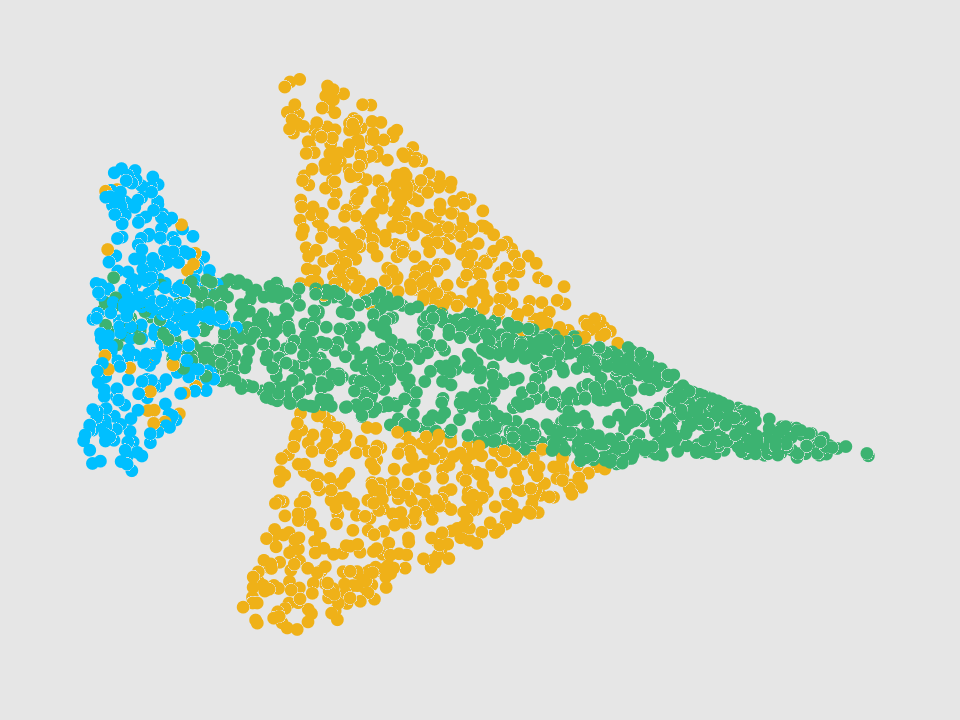} \\
        \vspace{-3mm}
        \includegraphics[width=1.0\columnwidth]{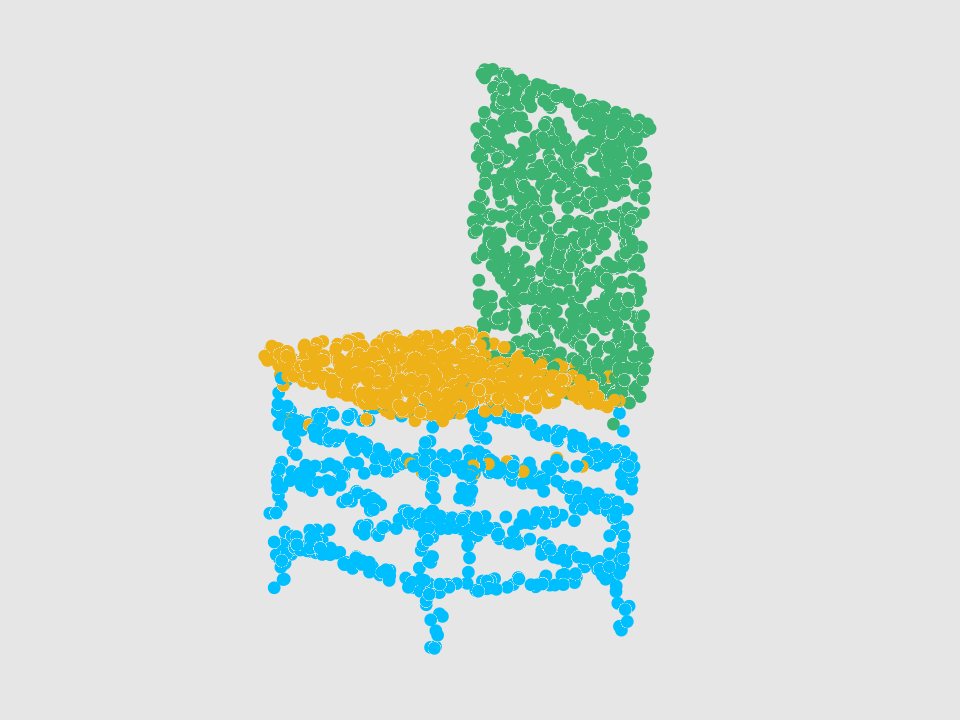} \\
        \vspace{-3mm}
        \includegraphics[width=1.0\columnwidth]{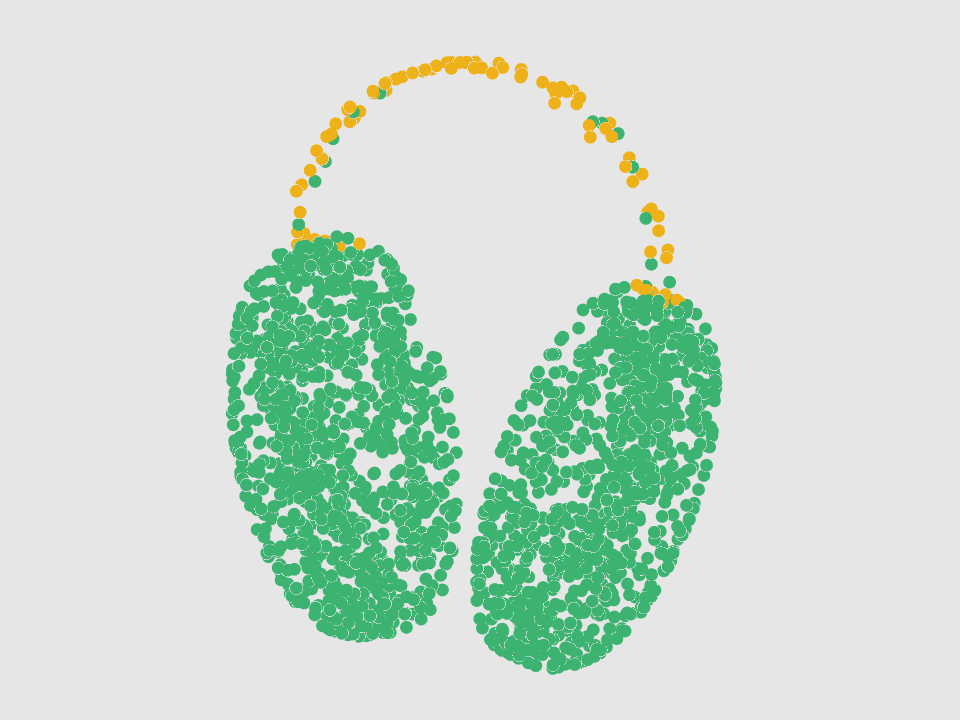}
        \end{minipage}
    }
    \subfloat[GT]{
        \begin{minipage}[b]{0.23\linewidth} 
        \includegraphics[width=1.0\columnwidth]{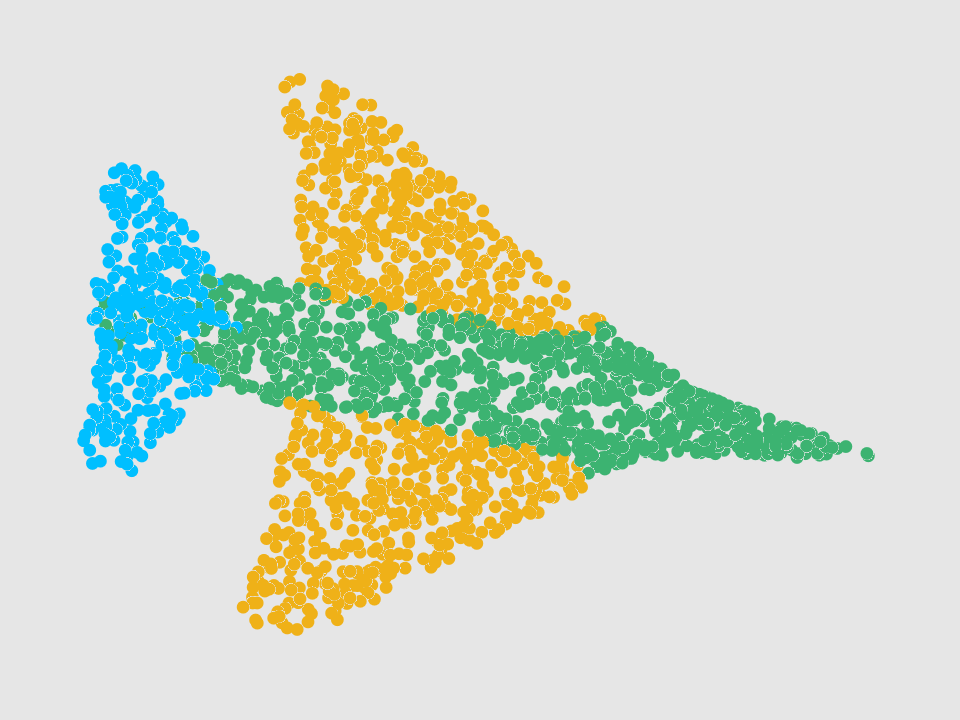} \\
        \vspace{-3mm}
        \includegraphics[width=1.0\columnwidth]{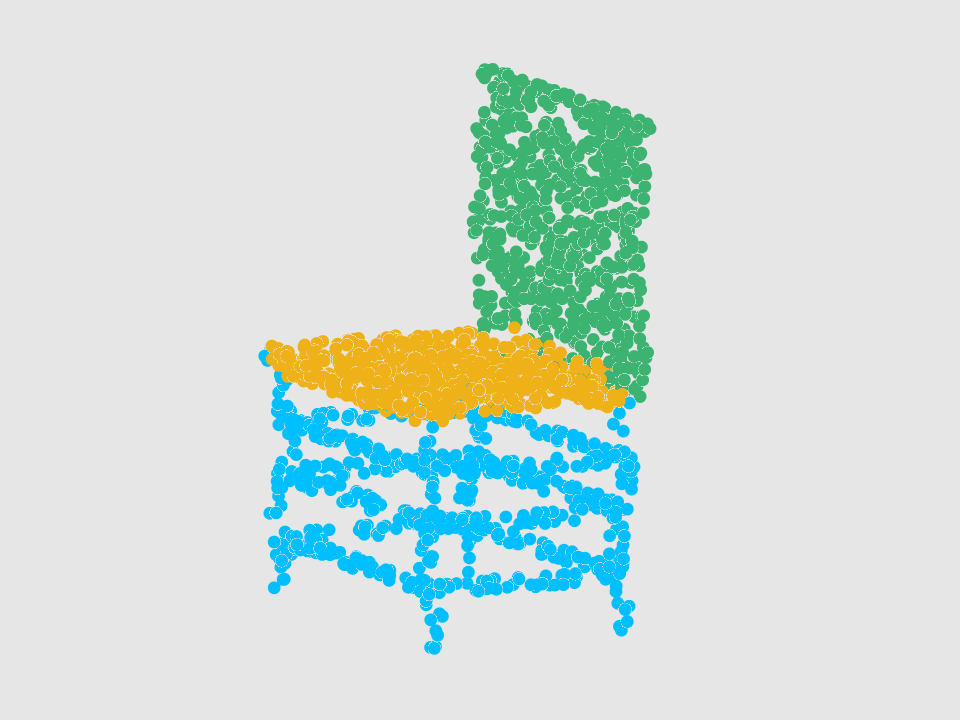} \\
        \vspace{-3mm}
        \includegraphics[width=1.0\columnwidth]{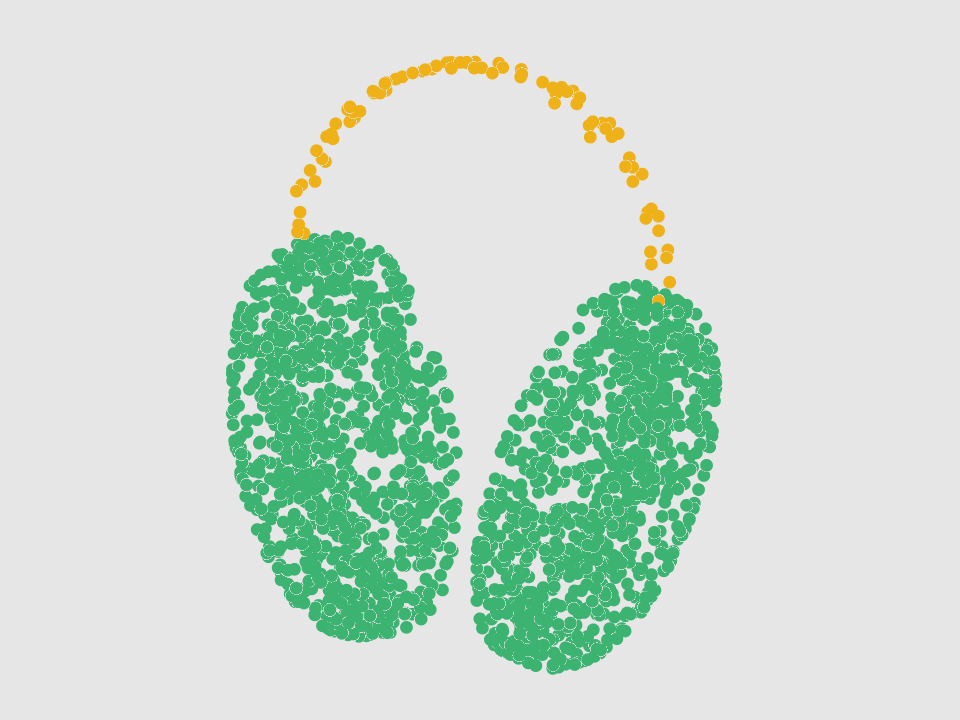}
        \end{minipage}
    }
    \caption{Visualization results on the ShapeNet-Part dataset using three different categories: airplane, chair and earphone. From left to right: PointNet++ (PN++), LSGCN, PointWavelet (Ours), and the ground truth (GT).}
    \label{fig.segmentation}
\end{figure}

\subsection{Point Cloud Semantic Segmentation}
Most recent state-of-the-art methods can obtain very good performances on synthetic datasets such as ModelNet40 and ShapeNet. To further evaluate our method on large-scale point cloud analysis, we show the experimental results of semantic indoor scene segmentation on the S3DIS dataset. 
In our experiment, we follow the same settings as PointNet~\cite{qi2017pointnet}, i.e., each room is first split into blocks and then each block is moved to the local coordinate system defined by its center. Each block uses 4,096 points during training process and all points are used for testing. For fair comparison, the results in S3DIS Area 5 are reported in Table~\ref{tbl.semantic_scene}, where the proposed method achieves either better or comparable performances with recent state-of-the-art methods.
\begin{table}[!ht]
    \caption{Point cloud semantic segmentation on S3DIS~\cite{armeni20163d}.}
    \centering
    \resizebox{0.6\columnwidth}{!}{
    \renewcommand\arraystretch{1.1}
    \begin{tabular}{lc}
    \toprule
    \multicolumn{1}{c}{Model} &\multirow{1}{*}{mIOU(\%)}\\
    \midrule
    \multirow{1}*{PointNet~\cite{qi2017pointnet}}
    &41.1  \\
    \multirow{1}*{DGCNN~\cite{wang2019dynamic}}
    &47.9  \\
    \multirow{1}*{PointAGCN~\cite{chen2018pointagcn}}
    &52.3  \\
    \multirow{1}*{PointCNN~\cite{li2018pointcnn}}
    &57.3  \\
    \multirow{1}*{KPConv~\cite{thomas2019kpconv}}
    &67.1  \\
    \multirow{1}*{Fast Point Trans\cite{park2022fast}}
    &70.3  \\
    \multirow{1}*{Point Trans~\cite{zhao2021point}}
    &70.4  \\
    \multirow{1}*{PointNeXt~\cite{qian2022pointnext}}
    &70.5  \\
    \multirow{1}*{PointMixer~\cite{choe2022pointmixer}}
    &\textbf{71.4}  \\
    \multirow{1}*{PointWavelet}
    &71.2  \\
    \multirow{1}*{PointWavelet-L}
    &71.3  \\
    \bottomrule
    \end{tabular}}
    \label{tbl.semantic_scene}
\end{table}

\begin{figure*}[!ht]
    \centering
    \subfloat[WF2]{
        \begin{minipage}[b]{0.32\linewidth} 
        \includegraphics[width=1\textwidth]{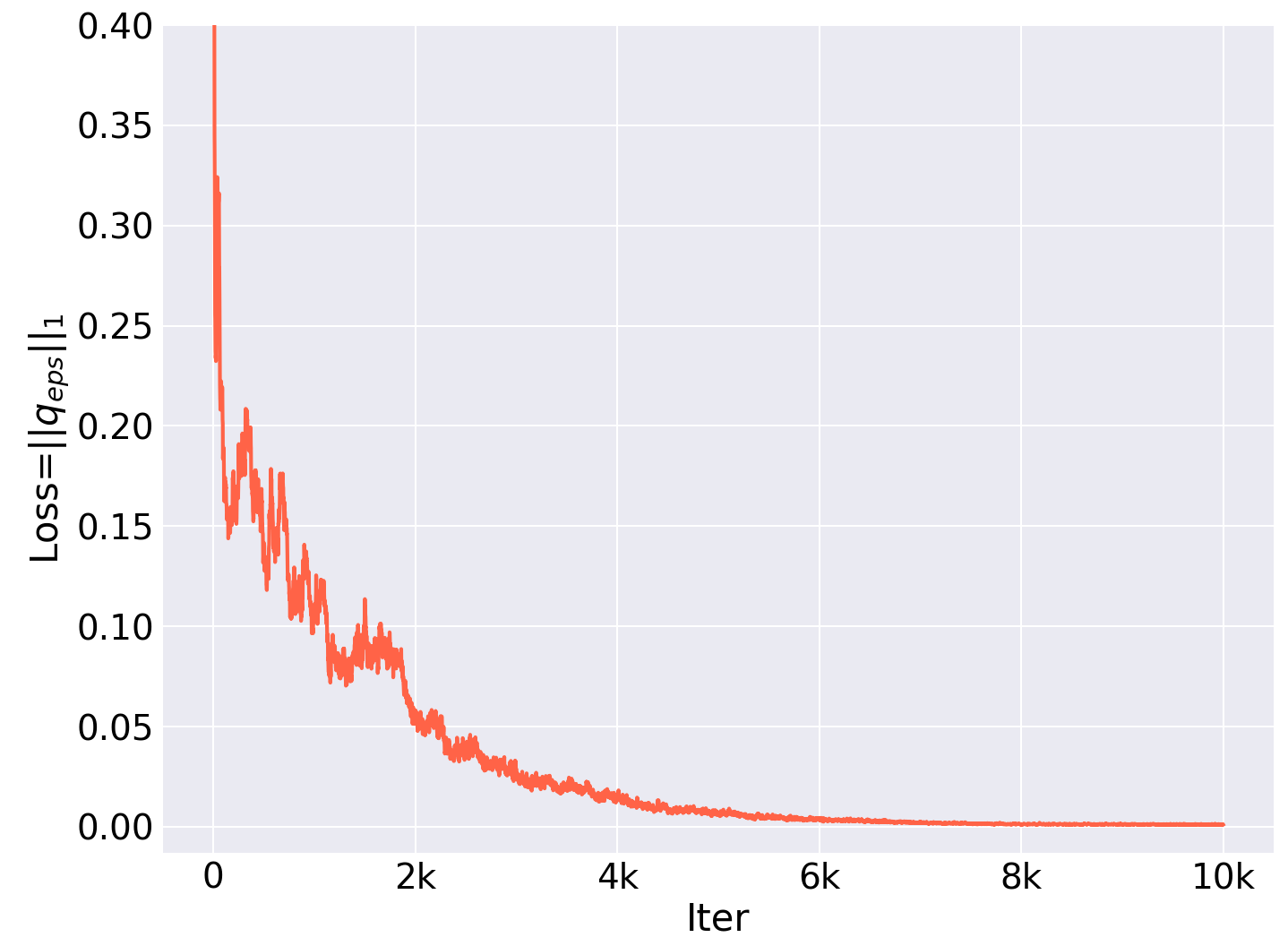}
        \end{minipage}
    }
    \subfloat[WF3]{
        \begin{minipage}[b]{0.32\linewidth} 
        \includegraphics[width=1\textwidth]{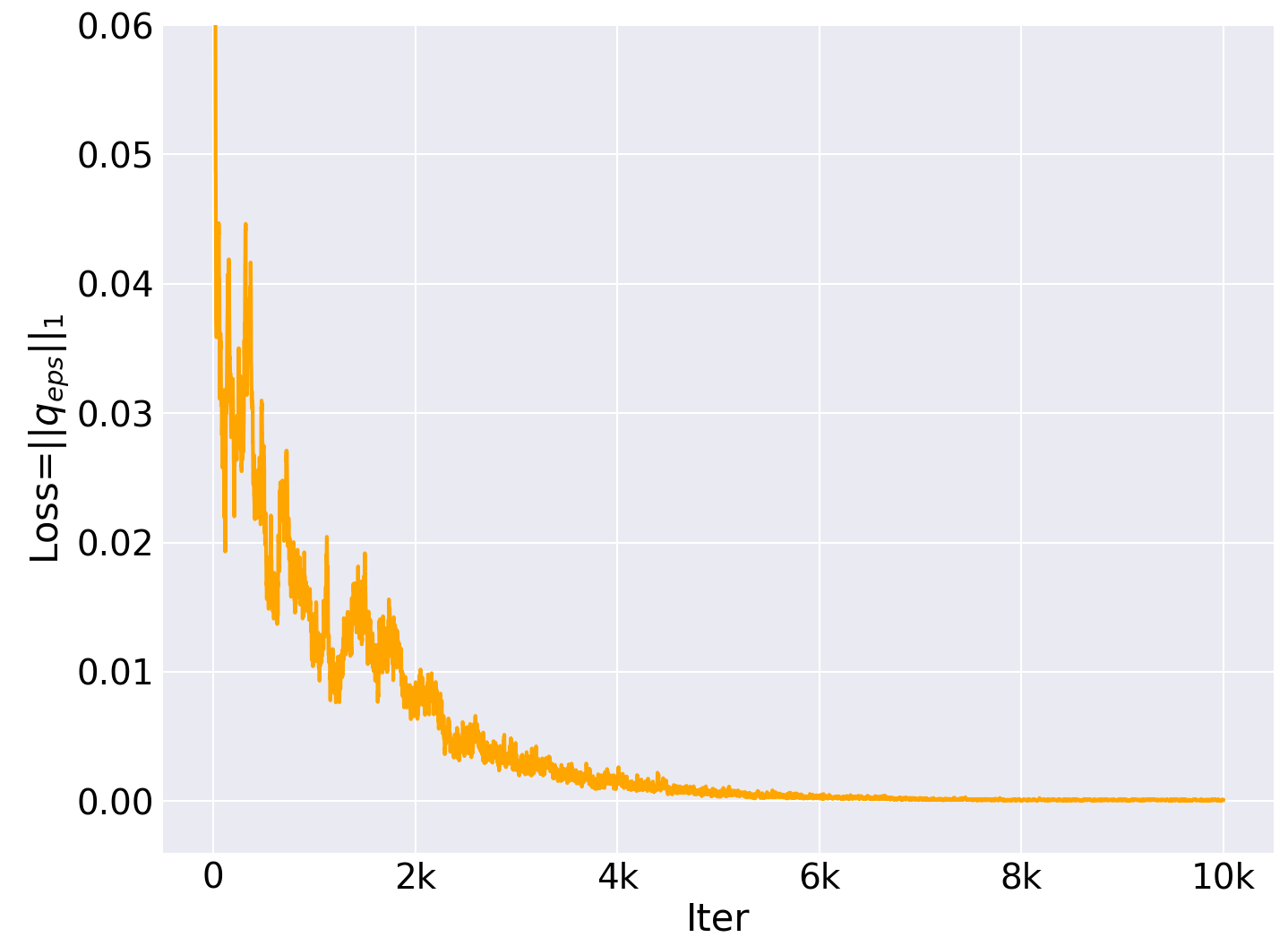}
        \end{minipage}
    }
    \subfloat[WF4]{
        \begin{minipage}[b]{0.32\linewidth} 
        \includegraphics[width=\textwidth]{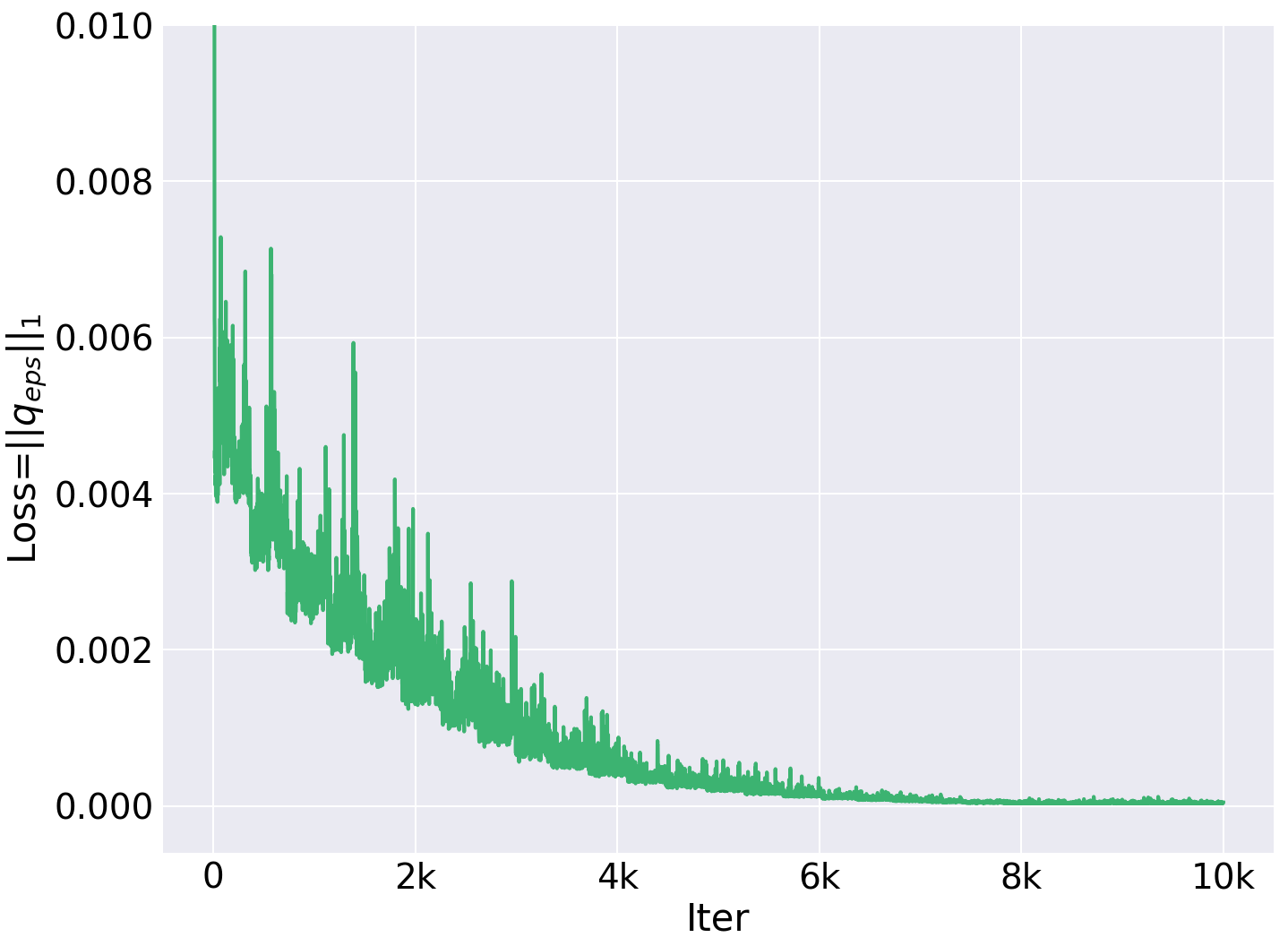}
        \end{minipage}
    }
    \caption{The $\|\bm{q_{eps}}\|_1$ loss curves of the three WaveletFormer (WF) layers. As shown, the three loss curves converge to zero rapidly.}
    \label{fig.loss_UUT}
\end{figure*}

\begin{table*}[!ht]
    \centering
    \caption{Point cloud classification results on ModelNet40 when adopting different learnable strategies.}
    \resizebox{0.8\textwidth}{!}{
    \begin{tabular}{lcccc p{1.0cm}<{\centering} p{1.0cm}<{\centering}}
    \toprule
    \multirow{2}{*}[-0.5ex]{Model} &\multirow{2}{*}[-0.5ex]{OA(\%)} &\multirow{2}{*}[-0.5ex]{mAcc(\%)} &\multirow{2}{*}[-0.5ex]{Params.(M)} &\multirow{2}{*}[-0.5ex]{FLOPs(G)}  &\multicolumn{2}{c}{Time(min)/epoch} \\
    \cmidrule(lr){6-7}
    & & & & &train &infer \\
    \midrule
    \multirow{1}*{PointWavelet}
    &94.1   &91.1 &54.70 &39.23 &18.5 &3.88 \\
    \multirow{1}*{PointWavelet-U}
    &92.5   &90.1 &60.71 &40.26 &3.05 &0.52 \\
    \multirow{1}*{PointWavelet-Che}
    &93.0   &90.2 &56.88 &39.85 &3.21 &0.63 \\
    \multirow{1}*{PointWavelet-L}
    &94.3   &91.1 &58.37 &39.16 &3.17 &0.52 \\
    \bottomrule
    \end{tabular}}
    \label{tbl.classification_UUT}
\end{table*}

\subsection{Ablation Studies}

\noindent\textbf{Loss Curve}. To better evaluate the learned orthogonal matrix $U$ in PointWavelet-L, we provide the $\|\bm{q_{eps}}\|_1$ loss curve of the three WaveletFormer layers when training on the ModelNet40 dataset (1024 points). When $\|\bm{q_{eps}}\|_1$ approaches zero, we then consider that the orthogonal matrix $U$ from Eq.~\eqref{U_construction_2} satisfies the condition of Theorem~\ref{Theorem}. As shown in Fig.~\ref{fig.loss_UUT}, the loss converges to zero after a few training iterations, indicating that the network can learn to represent the local graph efficiently. In addition, we are also interesting in how the $\|\bm{q_{eps}}\|_1$ regularization influences the performance of the network. To unveil this, we conduct another experiment to train PointWavelet-L without the $\|\bm{q_{eps}}\|_1$ loss (only $\mathcal{L}_{task}$), where we find that the network can also achieve a comparable performance, that is, 94.2\% OA on ModelNet40~\cite{wu20153d}. \\

\noindent\textbf{Different Strategies}. To validate the learnable orthogonal matrix, we not only compare the efficiency between PointWavelet and PointWavelet-L, but also another two strategies to avoid the high computational eigendecomposition operation as follows: \textbf{1) the first strategy} is to learn an orthogonal matrix directly by initializing a trainable matrix $U = [\bm{u_1}, \bm{u_2}, ..., \bm{u_n}]^\top \in \mathbb{R}^{n \times n}$ for which $u_1=(c, c, ..., c)^\top \in \mathbb{R}^{n}$ with $c \neq 0$ and enforcing $UU^\top=I$ during the training. If $\| I - UU^\top \|_F^2 < \epsilon$, where $\|\cdot \|_F$ indicates the matrix Frobenius norm, we then treat $U$ as the orthogonal matrix. We refer to this strategy as PointWavelet-U and the loss function is then defined as:
\begin{equation}
\mathcal{L} = \mathcal{L}_{task} + \beta \sum_{i}\| I - U_{i}U_{i}^\top \|_F^2, 
\end{equation}
where $\mathcal{L}_{task}$ is the loss on specific task such as classification or segmentation, and $i$ is the index of WaveletFormer layer with learnable orthogonal matrices. If not otherwise stated, we use $\beta=1.0$ in our experiments; 
\textbf{2) the second strategy}, which aims to bypass high-demanding computation, is to use truncated Chebyshev polynomials to approximate the spectral filtering~\cite{hammond2011wavelets,defferrard2016convolutional,kipf2016semi}. Specifically, the Chebyshev polynomials of the first kind is defined as
\begin{equation}
T_k(x) = cos(k \, arccosx), 
\end{equation}
and then we have the recurrence relation $T_0(x) = 1, T_1(x) = x, \cdots, T_{k+1}(x) = 2xT_k(x) - T_{k-1}(x)$. Given the input feature $f^(in)$ and Laplacian matrix $L$, the spectral graph convolution can be formulated as:
\begin{equation}
f^{(out)} = \sum_{k}\theta_k T_k(L) f^{(in)}, 
\end{equation}
where $\theta_k$ indicates the $k$-th trainable Chebyshev coefficient and $T_k(\cdot)$ is the Chebyshev polynomial of order $k$. We refer to this strategy as PointWavelet-Che and the loss function for PointWavelet-Che is $\mathcal{L}_{task}$.

\begin{figure*}[!ht]
    \centering
    \includegraphics[width=0.235\textwidth]{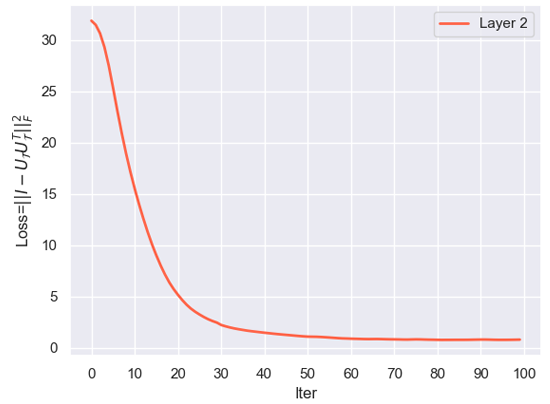}
    \hspace{1mm}
    \includegraphics[width=0.220\textwidth]{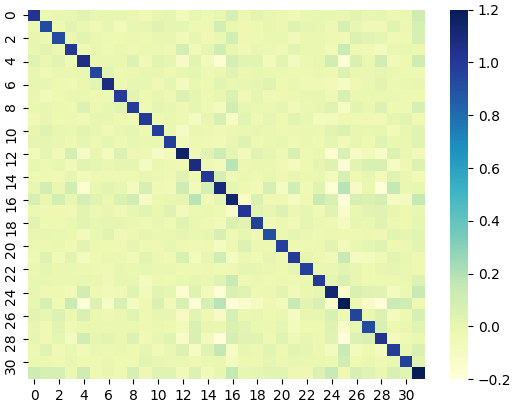}
    \hspace{1mm}
    \includegraphics[width=0.220\textwidth]{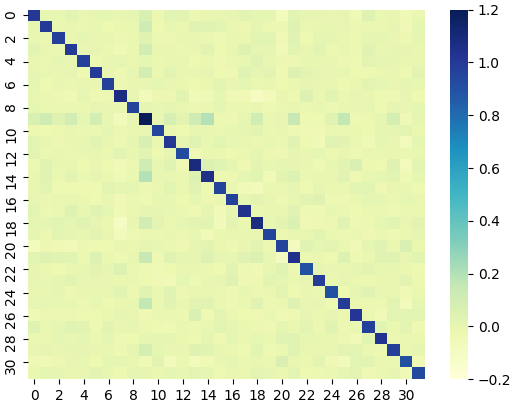}
    \hspace{1mm}
    \includegraphics[width=0.220\textwidth]{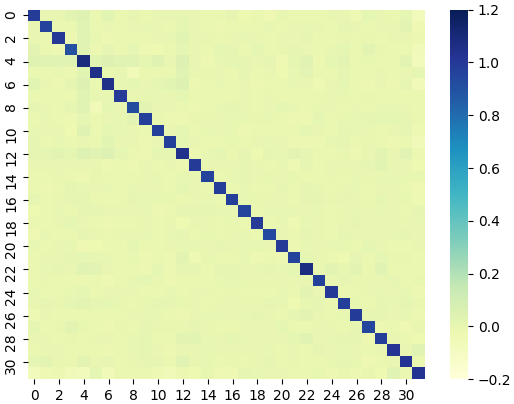} \\
    \vspace{2mm}
    \includegraphics[width=0.235\textwidth]{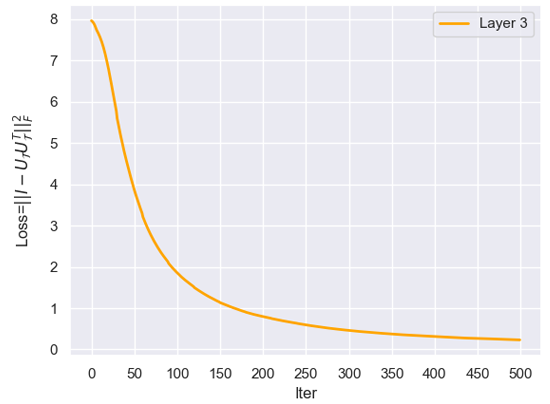}
    \hspace{1mm}
    \includegraphics[width=0.220\textwidth]{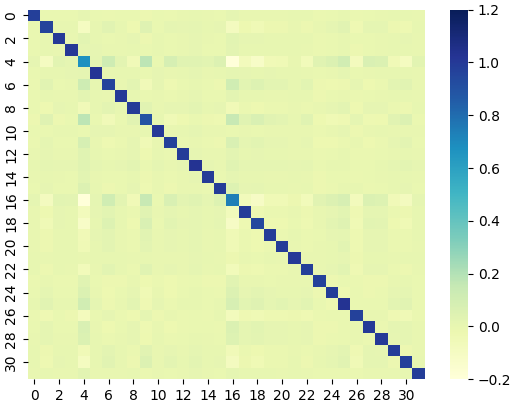}
    \hspace{1mm}
    \includegraphics[width=0.220\textwidth]{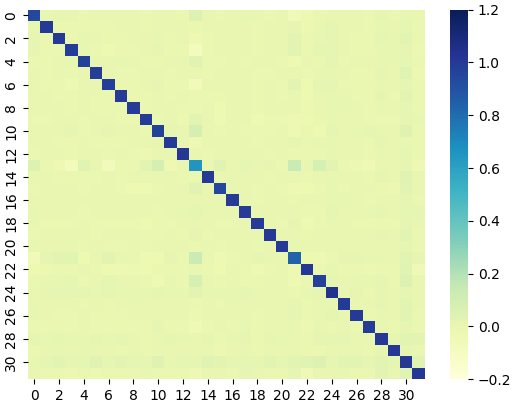}
    \hspace{1mm}
    \includegraphics[width=0.220\textwidth]{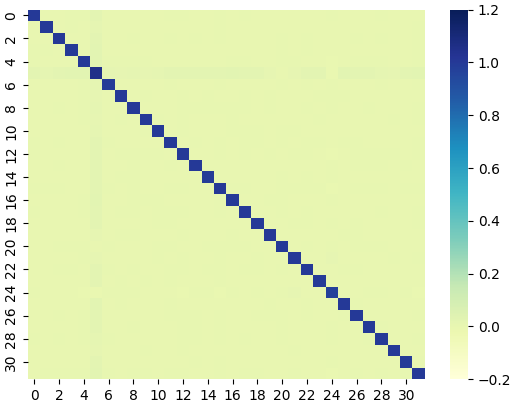} \\
    \vspace{2mm}
    \includegraphics[width=0.235\textwidth]{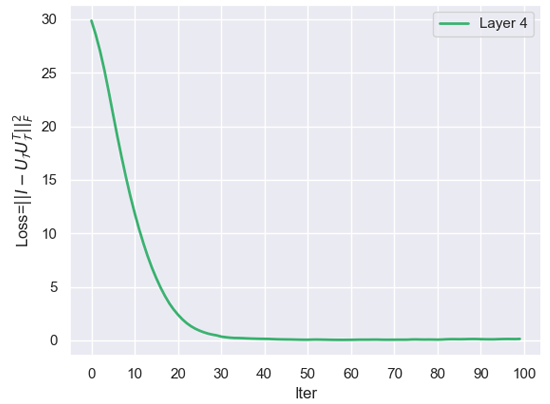}
    \hspace{1mm}
    \includegraphics[width=0.220\textwidth]{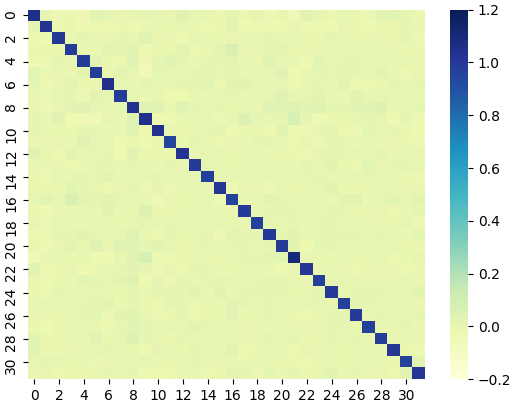}
    \hspace{1mm}
    \includegraphics[width=0.220\textwidth]{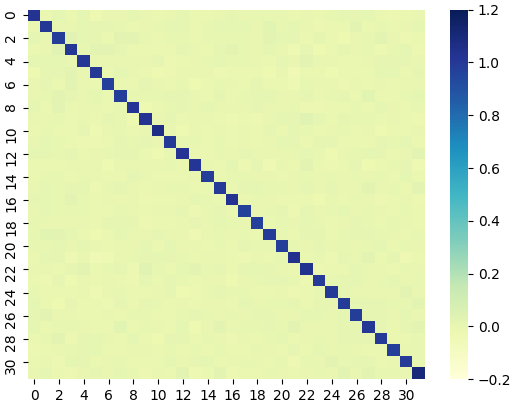}
    \hspace{1mm}
    \includegraphics[width=0.220\textwidth]{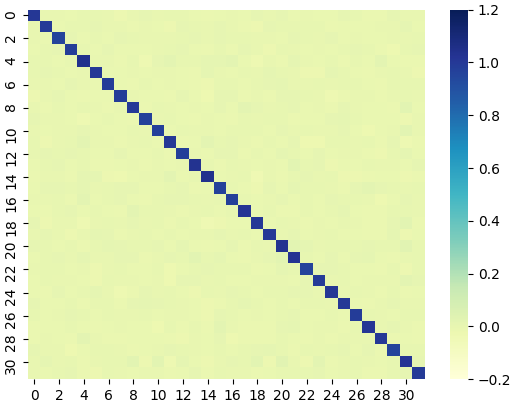}
    \caption{Visualization results of PointWavelet-U. From top to bottom: layer 2, layer 3 and layer 4; From left to right: the $\|I - UU^\top\|_F^2$ curve, the matrix product $UU^\top$ at epoch=10, 50, and 100, respectively. The corresponding quantitative results are presented in Table~\ref{tbl.UUT_deviation}. As shown, the $UU^\top$ converges to the identity matrix $I$ very fast.}
    \label{fig.UUT_directly}
\end{figure*}

For comparison, we report the classification results on ModelNet40 (1024 points). As shown in Table~\ref{tbl.classification_UUT}, we find that: 1) PointWavelet is much slow than the other three methods; 2) Though PointWavelet-U and PointWavelet-Che can accelerate the network training, they achieve inferior performances; and 3) Our PointWavelet-L achieves a better tradeoff between the speed and the performance. Additionally, we visualize the matrix product $UU^\top$ from the last three layers of PointWavelet-U and the loss curves in Fig.~\ref{fig.UUT_directly}.
In Table~\ref{tbl.UUT_deviation}, we evaluate the deviation between the identity matrix $I$ and the matrix product $UU^\top$ for the last three layers in the classification network at different training epochs.

\begin{table}[!ht]
    \caption{The deviation between the identity matrix and the matrix product on ModelNet40.}
    \centering
    \resizebox{0.9\columnwidth}{!}{
    \renewcommand\arraystretch{1.0}
    \begin{tabular}{cccc}
    \toprule
    \multirow{1}{*}{$\|I - UU^\top\|_F^2$} &\multirow{1}{*}{Epoch=10} &\multirow{1}{*}{Epoch=50} &\multirow{1}{*}{Epoch=100} \\
    \midrule
    \multirow{1}*{Layer 2}
    &2.9153    &0.9625    &0.2487 \\
    \multirow{1}*{Layer 3}
    &1.2496    &0.0818    &0.0326 \\
    \multirow{1}*{Layer 4}
    &0.2708    &0.1391    &0.0947 \\
    \bottomrule
    \end{tabular}}
    \label{tbl.UUT_deviation}
\end{table}

\noindent\textbf{Different Wavelet Scales}. When using graph wavelet transform, we need to choose the scale $J$ of wavelet functions. Therefore, we investigate how different $J$ influence the network performance and perform point cloud classification experiments on ModelNet40 (1024 points).
In Table~\ref{tbl.ablation}, we present the quantitative comparisons of PointWavelet-L, including testing overall accuracy (OA), network parameters (Params.), and floating point operations (FLOPs). In practice, we find that $J=5$ is a proper tradeoff in most cases. In addition, the large $J$ does not affect the accuracy drastically, because the low band usually dominates the frequency analysis in graph signal processing \cite{sandryhaila2013discrete,sandryhaila2014big}.

\begin{table}[!ht]
    \caption{The influence of different scales in graph wavelet transform.}
    \centering
    \resizebox{0.9\columnwidth}{!}{
    \renewcommand\arraystretch{1.0}
    \begin{tabular}{lccccc}
    \toprule
    \multirow{1}{*}{Scale} &\multirow{1}{*}{$J=3$} &\multirow{1}{*}{$J=5$} &\multirow{1}{*}{$J=7$} &\multirow{1}{*}{$J=9$} &\multirow{1}{*}{$J=11$} \\
    \midrule
    \multirow{1}*{OA(\%)}
    &93.82    &94.23    &94.05    &93.89    &93.96  \\
    \multirow{1}*{Params.(M)}
    &54.82     &58.37     &61.92     &65.48     &69.03   \\
    \multirow{1}*{FLOPs(G)}
    &24.72        &39.16        &53.61        &68.06        &82.53 \\
    \bottomrule
    \end{tabular}}
    \label{tbl.ablation}
\end{table}

\begin{figure*}[!ht]
\centering
    \subfloat[Mexican hat wavelets]{
        \begin{minipage}[b]{0.49\linewidth} 
        \includegraphics[width=1\textwidth]{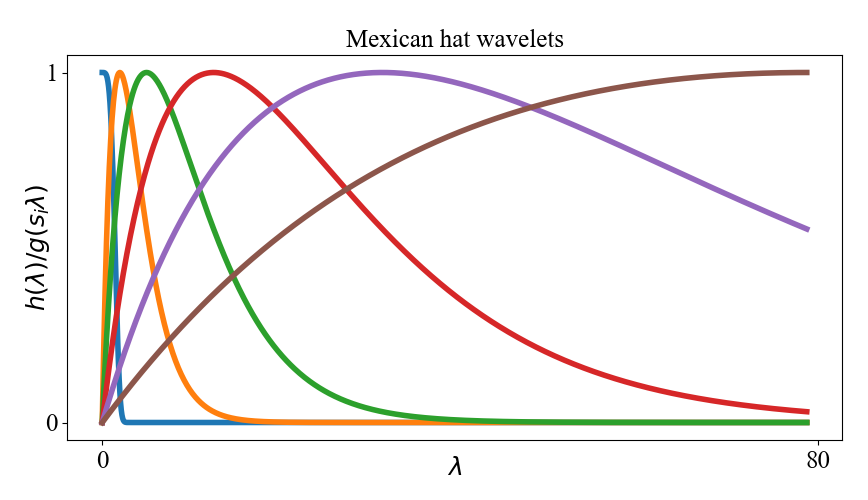}
        \end{minipage}
    }
    \subfloat[Meyer wavelets]{
        \begin{minipage}[b]{0.49\linewidth} 
        \includegraphics[width=1\textwidth]{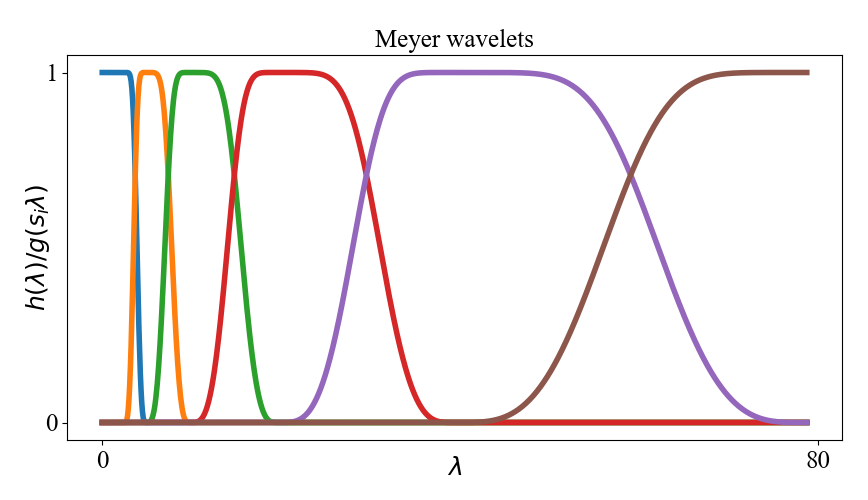}
        \end{minipage}
    }
    \caption{Frequency response of different wavelets. As shown, the wavelet functions are always stacked on the low frequency band of the spectrum.}
    \label{fig.different_wavelets}
\end{figure*}

\noindent\textbf{Different Wavelet Functions}. If not otherwise stated, we use Mexican hat wavelet in our experiments, where $h(\lambda) = e^{-x^4}$ and $
g(\lambda) = xe^{-x}$. To further demonstrate the effectiveness of our method, we also perform experiments using another popular wavelet, Meyer wavelet~\cite{vermehren2015close}, defined as
\begin{equation}
\label{meyer}
\begin{array}{l}
h(\lambda) =
\begin{cases}
\frac{1}{\sqrt{2\pi}} &\text{if} \; \lambda \leq \frac{2\pi}{3} \\
\frac{1}{\sqrt{2\pi}}cos[\frac{\pi}{2} v(\frac{3\lambda}{2\pi} - 1)] &\text{if} \; \frac{2\pi}{3} \leq \lambda \leq \frac{4\pi}{3} \\
0 &\text{otherwise}
\end{cases}
\end{array}
\end{equation}

\begin{equation}
\begin{array}{l}
g(\lambda) =
\begin{cases}
\frac{1}{\sqrt{2\pi}}sin[\frac{\pi}{2} v(\frac{3\lambda}{2\pi} - 1)]e^{\frac{\lambda}{2}j} &\text{if} \; \frac{2\pi}{3} \leq \lambda \leq \frac{4\pi}{3} \\
\frac{1}{\sqrt{2\pi}}cos[\frac{\pi}{2} v(\frac{3\lambda}{4\pi} - 1)]e^{\frac{\lambda}{2}j} &\text{if} \; \frac{4\pi}{3} \leq \lambda \leq \frac{8\pi}{3} \\
0 &\text{otherwise}
\end{cases}
\end{array}
\end{equation}
where 
\begin{equation}
\begin{array}{l}
v(x) = 
\begin{cases}
0 &\text{if} \; x < 0 \\
x &\text{if} \; 0 \leq x \leq 1 \\
1 &\text{if} \; x > 1
\end{cases}
\end{array}
\end{equation}

To show the influence of different wavelet functions on the proposed method, i.e., Mexican hat wavelets and Meyer wavelets, we perform experiments on ModelNet40 using 1024 points. As shown in Table~\ref{tbl.classification.MN40.Meyer}, we find that the proposed method can achieve comparable performance, suggesting the robustness of the proposed method to different wavelet functions. Additionally, we also show the frequency responses ($J=5$) of Mexican hat wavelets and Meyer wavelets in Fig.~\ref{fig.different_wavelets}.\\

\begin{table}[!ht]
    \caption{Point cloud classification using Mexican hat wavelets and Meyer wavelets. $*$ indicates using Meyer wavelets. }
    \centering
    \resizebox{0.6\columnwidth}{!}{
    \renewcommand\arraystretch{1.1}
    \begin{tabular}{lcc}
    \toprule
    \multicolumn{1}{c}{Model} &\multirow{1}{*}{OA(\%)} &\multirow{1}{*}{mAcc(\%)}  \\
    \midrule
    \multirow{1}*{PointWavelet}
    &94.1   &91.1  \\
    \multirow{1}*{PointWavelet-L}
    &94.3   &91.1  \\  
    \hline
    \multirow{1}*{PointWavelet$*$}
    &94.0   &90.9  \\
    \multirow{1}*{PointWavelet-L$*$}
    &94.1   &90.9  \\
    \bottomrule
    \end{tabular}}
    \label{tbl.classification.MN40.Meyer}
\end{table}

\noindent\textbf{Learnable Orthogonal Matrix}. In our paper, we devise a learnable orthogonal matrix $U$ instead of computing the eigendecomposition of Laplacian matrix directly to speed up the network training. To validate the effectiveness of this learning strategy, in this section we extend the usage of the learnable orthogonal matrix to graph Fourier transform. As LSGCN~\cite{wang2018local} uses spectral graph convolution for point cloud feature learning, we replace the traditional graph Fourier transform in LSGCN with its learnable counterpart, termed as LSGCN-L. Thus, LSGCN-L denotes graph Fourier transform by network learning according to the Theorem~\ref{Theorem}. We present the classification results on ModelNet40 (1024 points) in Table~\ref{tbl.classification.MN40.extentd}, where LSGCN-L achieves better performance than LSGCN with a much faster training/inference speed. Note that, as LSGCN is implemented by TensorFlow, thus this experiment is also conducted using TensorFlow.

\begin{table}[!ht]
    \caption{Classification results using learnable Fourier transform.}
    \centering
    \resizebox{0.7\columnwidth}{!}{
    \renewcommand\arraystretch{1.1}
    \begin{tabular}{lc p{1.0cm}<{\centering} p{1.05cm}<{\centering}}
    \toprule
    \multirow{2}{*}[-0.5ex]{Model} &\multirow{2}{*}[-0.5ex]{OA(\%)}  &\multicolumn{2}{c}{Time(min)/epoch} \\
    \cmidrule(lr){3-4}
    & &train &infer\\
    \midrule
    \multirow{1}*{LSGCN}
    &91.6   &8.06   &1.92   \\
    \multirow{1}*{LSGCN-L}
    &91.8   &1.42   &0.14   \\
    \bottomrule
    \end{tabular}}
    \label{tbl.classification.MN40.extentd}
\end{table}

\subsection{Visualization}

Our method extracts point features in the spectral domain instead of the spatial domain, and thus it is helpful to investigate the difference between the features of two domains. Considering that we use the set abstraction (SA) layer from PointNet++~\cite{qi2017pointnet++} at the feature extraction stage, thus we take PointNet++ for comparison. We use t-SNE~\cite{van2008visualizing} to visualize the embedded features of PointNet++ and PointWavelet in Fig.~\ref{fig.feature_visual}. All high dimensional features are taken from the output of grouping and sampling block (SA layer in PointNet++ and the combination of SA and WF layer in our work). As shown, the two domains of features are discriminative against each other, which indicates the spectral graph convolution network always learns different representations from its spatial counterpart.
\begin{figure}[!ht]
    \centering
    \includegraphics[width=\linewidth]{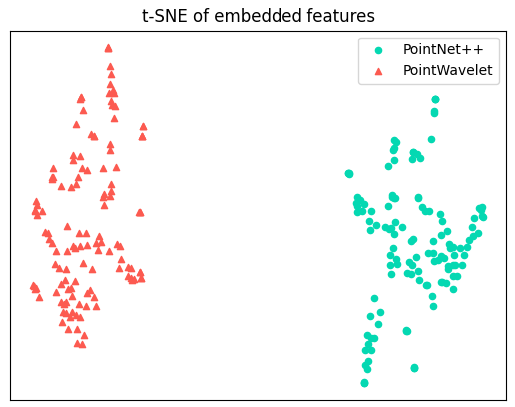}
    \caption{Visualization of the embedded features from PointNet++ and PointWavelet using t-SNE.}
    \label{fig.feature_visual}
\end{figure}

\section{Conclusion}

In this paper, we introduce a new method, PointWavelet, for point cloud analysis in the spectral domain. By utilizing the graph wavelet transform, we propose a new WaveletFormer layer to explore the relationships between spectral features at different scales. To avoid the time-consuming eigendecomposition operation, we adopt a novel learnable strategy, that is, we encourage the network to learn to represent the local graph by training a learnable orthogonal matrix. Without directly calculating the Laplacian matrix, this method significantly facilitates the overall training process. Overall, the proposed method can well capture local structural information from graph data, which presents an advance in learning-based spectral analysis for computer vision applications. 

\bibliographystyle{IEEEtran}

\bibliography{PointWavelet}

\vspace{-30pt}
\begin{IEEEbiography}[{\includegraphics[width=1in,height=1.25in,clip,keepaspectratio]{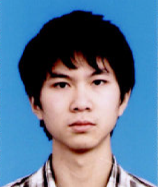}}]{Cheng Wen} received a B.S. in Mechanical Engineering and Automation from Beijing University of Posts and Telecommunications, and a M.S. in Biomedical Engineering from Tsinghua University. Now he is a Ph.D. candidate in Computer Science from The University of Sydney. His research interests include computer vision, machine learning, and deep learning. Currently, he mainly focuses on 3D computer vision, including 3D point cloud classification and segmentation.
\end{IEEEbiography}
\vspace{-30pt}
\begin{IEEEbiography}[{\includegraphics[width=1in,height=1.25in,clip,keepaspectratio]{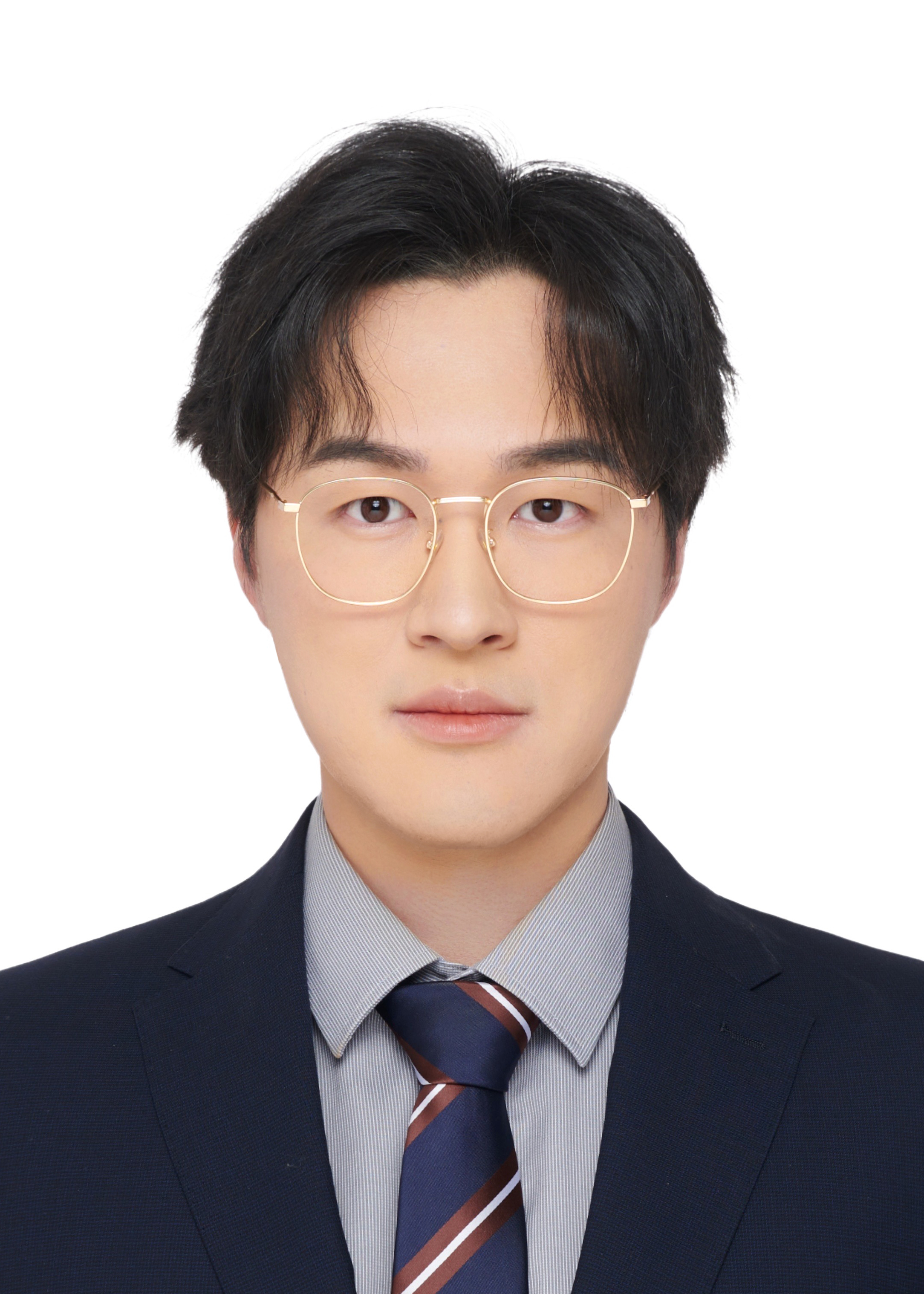}}]{Jianzhi Long} received a B.S. degree in electrical engineering from University of Illinois at Urbana-Champaign in 2020, and a M.S. degree in electrical engineering from University of California San Diego in 2022. Now he is a Ph.D. candidate in Computer Science from The University of Sydney. His research interests include computer vision and computer graphics. 
\end{IEEEbiography}
\vspace{-30pt}
\begin{IEEEbiography}[{\includegraphics[width=1in,height=1.25in,clip,keepaspectratio]{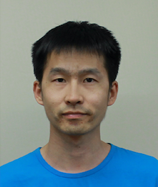}}]{Baosheng Yu} received a B.E. in Computer Science from the University of Science and Technology of China in 2014, and a Ph.D. in Computer Science from The University of Sydney in 2019. He is currently a Research Fellow in the School of Computer Science at the University of Sydney, NSW, Australia. His research interests include computer vision, machine learning, and deep learning. He has authored/co-authored more than 20 publications on top-tier international conferences and journals, including IEEE TPAMI, CVPR, ICCV and ECCV.
\end{IEEEbiography}
\vspace{-30pt}
\begin{IEEEbiography}[{\includegraphics[width=1in,height=1.25in,clip]{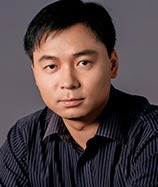}}]{Dacheng Tao (F’15)} is the President of the JD Explore Academy and a Senior Vice President of JD.com. He is also an advisor and chief scientist of the digital science institute in the University of Sydney. He mainly applies statistics and mathematics to artificial intelligence and data science, and his research is detailed in one monograph and over 200 publications in prestigious journals and proceedings at leading conferences. He received the 2015 Australian Scopus-Eureka Prize, the 2018 IEEE ICDM Research Contributions Award, and the 2021 IEEE Computer Society McCluskey Technical Achievement Award. He is a fellow of the Australian Academy of Science, AAAS, and ACM.
\end{IEEEbiography}

\begin{strip}
\begin{center}
    \Huge
    PointWavelet: Learning in Spectral Domain for 3D Point Cloud Analysis (Supplementary Material)
    \\[40pt]
\end{center}
\setcounter{section}{0}
\setcounter{theorem}{0}
\setcounter{equation}{0}

\section{Proof of Theorem 1}
\noindent

\begin{theorem}
Given a set of vertices $\mathcal{V}$ with $|\mathcal{V}| = n$, there exists one orthogonal matrix $U = [\bm{u_1}, \bm{u_2}, ..., \bm{u_n}] \in \mathbb{R}^{n \times n}$ for which $\bm{u_1}=(c, c, ..., c)^\top \in \mathbb{R}^{n}$ with $c \neq 0$, and a set of $\lambda_i$ satisfying $0 = \lambda_1 \leq \lambda_2 ... \leq \lambda_n$, such that $U \Lambda U^\top \in \mathbb{R}^{n \times n}$ forms the Laplacian matrix of local graph $\mathcal{G}$ , where $\Lambda = \mathop{\mathrm{diag}}(\lambda_1, \lambda_2, ..., \lambda_n)$.
\end{theorem}

\begin{proof}
Given an arbitrary orthogonal matrix $\tilde{U} = [\bm{\tilde{u}_1}, \bm{\tilde{u}_2}, ..., \bm{\tilde{u}_n}] \in \mathbb{R}^{n \times n}$ for which $\bm{\tilde{u}_1}=[\tilde{c}, \tilde{c}, ..., \tilde{c}]^T \in \mathbb{R}^{n}$ with $\tilde{c} \neq 0$, and a set of $\tilde{\lambda}_i, i=1, 2, ..., n$ with $0 = \tilde{\lambda}_1 \leq \tilde{\lambda}_2 ... \leq \tilde{\lambda}_n$, we have $\tilde{L} = \tilde{U}\tilde{\Lambda} \tilde{U}^\top$, where $\tilde{\Lambda} = diag(\tilde{\lambda}_1, \tilde{\lambda}_2, ..., \tilde{\lambda}_n) \in \mathbb{R}^{n \times n}$. For $\tilde{U}^{-1} = \tilde{U}^\top$, thus $\tilde{U} \tilde{\Lambda} \tilde{U}^\top = \tilde{U} \tilde{\Lambda} \tilde{U}^{-1}$ is the eigen-decomposition of $\tilde{L}$. Apparently, $\tilde{L}$ is symmetric and positive semi-definite. Let $\tilde{L}$ to be
\begin{equation}
\tilde{L} = \left[
\begin{array}{cccc}
l_{11} & l_{12} & \cdots & l_{1n} \\
l_{21} & l_{22} & \cdots & l_{2n} \\
\vdots & \vdots & \ddots & \vdots \\
l_{n1} & l_{n2} & \cdots & l_{nn}
\end{array}
\right]
\end{equation}
and we have
$\tilde{L}\bm{\tilde{u}_1} = \tilde{\lambda}_1 \bm{\tilde{u}_1} =  0, \text{which means} \sum_{j=1}^{n}l_{ij} = 0, i=1, 2, \cdots, n.
$ \\ 
If we suppose that there is a matrix $\tilde{A} \in \mathbb{R}^{n \times n}$ and a diagonal matrix $\tilde{D} \in \mathbb{R}^{n \times n}$, where
\begin{equation}
\tilde{A} = \left[
\begin{array}{cccc}
a_{11} & a_{12} & \cdots & a_{1n} \\
a_{21} & a_{22} & \cdots & a_{2n} \\
\vdots & \vdots & \ddots & \vdots \\
a_{n1} & a_{n2} & \cdots & a_{nn}
\end{array}
\right]
\end{equation}
and 
\begin{equation}
\tilde{D} = \left[
\begin{array}{cccc}
\sum_{i=1}^{n}a_{1i} & & &  \\
& \sum_{i=1}^{n}a_{2i} & &  \\
& & \ddots &  \\
& & & \sum_{i=1}^{n}a_{ni}
\end{array}
\right].
\end{equation}
Let $\tilde{L} = \tilde{D} - \tilde{A}$, we then have
\begin{small}
\begin{equation}
\left[
\begin{array}{cccc}
l_{11} & l_{12} & \cdots & l_{1n} \\
l_{21} & l_{22} & \cdots & l_{2n} \\
\vdots & \vdots & \ddots & \vdots \\
l_{n1} & l_{n2} & \cdots & l_{nn}
\end{array}
\right]
=
\left[
\begin{array}{cccc}
\sum_{i=1}^{n}a_{1i} & & &  \\
& \sum_{i=1}^{n}a_{2i} & &  \\
& & \ddots &  \\
& & & \sum_{i=1}^{n}a_{ni}
\end{array}
\right]
-
\left[
\begin{array}{cccc}
a_{11} & a_{12} & \cdots & a_{1n} \\
a_{21} & a_{22} & \cdots & a_{2n} \\
\vdots & \vdots & \ddots & \vdots \\
a_{n1} & a_{n2} & \cdots & a_{nn}
\end{array}
\right].
\end{equation}
\end{small}
Taken the first row for example, we have \\
\begin{equation}
\left\{
\begin{array}{lcl}
l_{11} &=& \sum_{i=1}^{n}a_{1i}- a_{11} \\
l_{12} &=& - a_{12} \\
&\vdots& \\
l_{1n} &=& - a_{1n}
\end{array}
\right.
\Longrightarrow
\left[
\begin{array}{cccc}
0 & 1 & \cdots & 1 \\
0 & -1 & \cdots & 0 \\
\vdots & \vdots & \ddots & \vdots \\
0 & 0 & \cdots & -1
\end{array}
\right]
\left[
\begin{array}{cccc}
a_{11} \\
a_{12} \\
\vdots \\
a_{1n}
\end{array}
\right]
=
\left[
\begin{array}{cccc}
l_{11} \\
l_{12} \\
\vdots \\
l_{1n}
\end{array}
\right]
\end{equation}
and
\begin{equation}
rank(
\left[
\begin{array}{ccccc}
0 & 1 & \cdots & 1 & l_{11} \\
0 & -1 & \cdots & 0 & l_{12} \\
\vdots & \vdots & \ddots & \vdots  &\vdots \\
0 & 0 & \cdots & -1 &l_{1n}
\end{array}
\right]
) = n - 1
\Longrightarrow
\left\{
\begin{array}{rcl}
a_{11} &=& 0 \\
a_{12} &=& -l_{12} \\
&\vdots& \\
a_{1n} &=& -l_{1n} \\
\end{array}
\right..
\end{equation}
Here, $a_{11}$ can be an arbitrary real number, we set $a_{11} = 0$ for simplicity. Therefore, we have
\begin{equation}
\tilde{A} = 
\left[
\begin{array}{cccc}
      0 & -l_{12} & \cdots & -l_{1n} \\
-l_{21} &       0 & \cdots & -l_{2n} \\
 \vdots &  \vdots & \ddots &  \vdots \\
-l_{n1} & -l_{n2} & \cdots & 0
\end{array}
\right]
\end{equation}
Note that, since $\tilde{L}$ is symmetric, $\tilde{A}$ is also symmetric. Apparently, we can treat $\tilde{A}$ as the adjacency matrix of $\mathcal{V}$. Let $\mathbb{H}(\mathcal{V}):=\{\bm{F}: \mathcal{V} \rightarrow \mathbb{R}\}$ be the space of real vertex functions. Since $\mathcal{V}$ is a finite set, any vertex function $\bm{F} \in \mathbb{H}(\mathcal{V})$ can be represented as a $n$-dimensional vector $\bm{F} \in \mathbb{R}^n$. Let
\begin{equation}
\bm{F}(\mathcal{V}) = 
\left\{
\begin{array}{ll}
F(x_i, x_i) = 0,       &i = 1, 2, \cdots, n \\
F(x_i, x_j) = -l_{ij}, &i, j = 1, 2, \cdots, n, \text{and} \: i \neq j
\end{array}
\right.
\end{equation}
and according to universal approximation theorem \cite{hornik1991approximation}, this function can be approximated by neural network. Thus the $U$ and $\lambda_l$ must exist, and $U = \tilde{U}$ and $\lambda_l = \tilde{\lambda}_l$ can be learned by parameters tuning.
\end{proof}

\section{Computation of $(\Psi^\top \Psi)^{-1}$}

Let $p(x) = h^2(x) + g^2(s_1x) + \cdots + g^2(s_Jx)$, we then have 
\begin{equation}
(\Psi^\top \Psi)^{-1} = U p_{\lambda}^{-1} U^\top,
\end{equation}
where $p_{\lambda}^{-1} = diag(p^{-1}(\lambda_1), p^{-1}(\lambda_2), \cdots, p^{-1}(\lambda_n)) \in \mathbb{R}^{n \times n}$. We present the details of the computation.

\begin{proof}
According to Equation (6) of the main paper, we have
\begin{equation}
\Psi = \left[
\begin{array}{c}
\Psi_0 \\ 
\Psi_{s_i} \\ 
\vdots \\ 
\Psi_{s_J}
\end{array}
\right]
\end{equation}
and
\begin{equation}
\Psi^\top \Psi = [\Psi_0^\top, \Psi_{s_i}^\top, \cdots, \Psi_{s_i}^\top]
\left[
\begin{array}{c}
\Psi_0 \\ 
\Psi_{s_i} \\ 
\vdots \\ 
\Psi_{s_J}
\end{array}
\right]
= \Psi_0^\top \Psi_0 + \Psi_{s_1}^\top \Psi_{s_1} + \cdots + \Psi_{s_J}^\top \Psi_{s_J}.
\end{equation}
Taken $\Psi_0^\top \Psi_0$ for example, we have
\begin{equation}
\begin{aligned}
\Psi_0^\top \Psi_0 
&= (U \left[
\begin{array}{cccc}
h(\lambda_1) & & &  \\
& h(\lambda_2) & &  \\
& & \ddots &  \\
& & & h(\lambda_n)
\end{array}
\right] U^\top)^\top U \left[
\begin{array}{cccc}
h(\lambda_1) & & &  \\
& h(\lambda_2) & &  \\
& & \ddots &  \\
& & & h(\lambda_n)
\end{array}
\right] U^\top \\
&= U \left[
\begin{array}{cccc}
h^2(\lambda_1) & & &  \\
& h^2(\lambda_2) & &  \\
& & \ddots &  \\
& & & h^2(\lambda_n)
\end{array}
\right] U^\top.
\end{aligned}
\end{equation}
Let $p(x) = h^2(x) + g^2(s_1x) + \cdots + g^2(s_Jx)$, we have \\
\begin{equation}
\Psi^\top \Psi = 
U \left[
\begin{array}{cccc}
p(\lambda_1) & & &  \\
& p(\lambda_2) & &  \\
& & \ddots &  \\
& & & p(\lambda_n)
\end{array}
\right] U^\top
\end{equation}
and
\begin{equation}
(\Psi^\top \Psi)^{-1} = 
U \left[
\begin{array}{cccc}
p^{-1}(\lambda_1) & & &  \\
& p^{-1}(\lambda_2) & &  \\
& & \ddots &  \\
& & & p^{-1}(\lambda_n)
\end{array}
\right] U^\top,
\end{equation}
where $p^{-1}(\lambda_i) = (h^2(\lambda_i) + g^2(s_1\lambda_i) + \cdots + g^2(s_J\lambda_i))^{-1}, i=1, 2, \cdots, n$.
\end{proof}

\section{The Orthogonality of Matrix Constructed from Vector}

In our main paper, we obtain the orthogonal matrix $U$ from a vector $\bm{q}$. Therefore, we further give a brief proof of its orthogonality as follows.
\begin{proof}
Given the normalized vector $\bm{q}=(q_1, q_2, ..., q_n)^\top \in \mathbb{R}^{n}$ with $||\bm{q}||_2 = 1$, then
\begin{equation}
U = [\bm{u_1}, \bm{u_2}, ..., \bm{u_n}] = 
\left[
\begin{array}{cccc}
q_{1}   & -q_{2}     & \cdots & -q_{n}    \\
q_{2}   & F(2, 2) + 1  & \cdots & F(2, n) \\
\vdots  & \vdots     & \ddots & \vdots    \\
q_{n}   & F(n, 2)  & \cdots & F(n, n)+1
\end{array}
\right],
\end{equation}
where
\begin{equation}
F(i, j) = \frac{q_i q_j}{\sum_{i=2}^{n} |q_{i}|^2} (q_{1} -1).
\end{equation}
If $U$ is orthogonal, that is, $UU^\top = I$, it is equivalent to
\begin{equation}
\left\{
\begin{array}{lcl}
\bm{u_i}^\top \bm{u_i} = 1, (i=1, 2, 3, ..., n)\\
\vspace{1mm}
\bm{u_i}^\top \bm{u_j} = 0, (i \neq j)\\
\end{array}
\right..
\end{equation}
Since $\bm{u_1}$ is different from $\bm{u_i} (i=2, 3, ..., n)$, we thus compute $\bm{u_1}^\top \bm{u_1}$, $\bm{u_1}^\top \bm{u_i}(i \neq 1)$, $\bm{u_i}^\top \bm{u_i}(i \neq 1)$ and $\bm{u_i}^\top \bm{u_j}(i,j \neq 1 \ \ \text{and} \ \ i \neq j)$ as follows.
\begin{equation}
\label{eq:case1}
\begin{aligned}
\bm{u_1}^\top \bm{u_1} &=
\left[
\begin{array}{c}
q_{1} \\ 
q_{2} \\ 
q_{3} \\
\vdots \\ 
q_{n}
\end{array}
\right]^\top
\left[
\begin{array}{c}
q_{1} \\ 
q_{2} \\ 
q_{3} \\
\vdots \\ 
q_{n}
\end{array}
\right]
= |q_{1}|^2 + |q_{2}|^2 + |q_{3}|^2 + \cdots + |q_{n}|^2 
= 1
\end{aligned}
\end{equation}

\begin{equation}
\label{eq:case2}
\begin{aligned}
\bm{u_1}^\top \bm{u_i} =&
\left[
\begin{array}{c}
q_{1} \\ 
q_{2} \\ 
\vdots \\ 
q_{i} \\
\vdots \\ 
q_{n}
\end{array}
\right]^\top
\left[
\begin{array}{c}
-q_{i} \\ 
F(2, i) \\
\vdots \\
F(i, i) + 1 \\
\vdots \\ 
F(n, i)
\end{array}
\right] 
=
\left[
\begin{array}{c}
q_{1} \\ 
q_{2} \\ 
\vdots \\ 
q_{i} \\
\vdots \\ 
q_{n}
\end{array}
\right]^\top
\left[
\begin{array}{c}
-q_{i} \\ 
\frac{q_2 q_i}{\sum_{i=2}^{n} |q_{i}|^2} (q_{1} -1) \\
\vdots \\
\frac{q_i q_i}{\sum_{i=2}^{n} |q_{i}|^2} (q_{1} -1) + 1 \\
\vdots \\ 
\frac{q_n q_i}{\sum_{i=2}^{n} |q_{i}|^2} (q_{1} -1)
\end{array}
\right] \\
=& -q_{1}q_{i} + \frac{q_2^2 q_i}{\sum_{i=2}^{n} |q_{i}|^2} (q_{1} -1) + \cdots + \frac{q_i^2 q_i}{\sum_{i=2}^{n} |q_{i}|^2} (q_{1} -1) + q_i + \cdots
+ \frac{q_n^2 q_i}{\sum_{i=2}^{n} |q_{i}|^2} (q_{1} -1) \\
=& -q_{1}q_{i} + q_{i}(q_1 - 1) +q_{i} \\
=& \ 0
\end{aligned}
\end{equation}

\clearpage

\begin{equation}
\label{eq:case3}
\begin{aligned}
\bm{u_i}^\top \bm{u_i} =&
\left[
\begin{array}{c}
-q_{i} \\ 
F(2, i) \\
\vdots \\
F(i, i) + 1 \\
\vdots \\ 
F(n, i)
\end{array}
\right]^\top
\left[
\begin{array}{c}
-q_{i} \\ 
F(2, i) \\
\vdots \\
F(i, i) + 1 \\
\vdots \\ 
F(n, i)
\end{array}
\right]
= 
\left[
\begin{array}{c}
-q_{i} \\ 
\frac{q_2 q_i}{\sum_{i=2}^{n} |q_{i}|^2} (q_{1} -1) \\
\vdots \\
\frac{q_i q_i}{\sum_{i=2}^{n} |q_{i}|^2} (q_{1} -1) + 1 \\
\vdots \\ 
\frac{q_n q_i}{\sum_{i=2}^{n} |q_{i}|^2} (q_{1} -1)
\end{array}
\right]^\top
\left[
\begin{array}{c}
-q_{i} \\ 
\frac{q_2 q_i}{\sum_{i=2}^{n} |q_{i}|^2} (q_{1} -1) \\
\vdots \\
\frac{q_i q_i}{\sum_{i=2}^{n} |q_{i}|^2} (q_{1} -1) + 1 \\
\vdots \\ 
\frac{q_n q_i}{\sum_{i=2}^{n} |q_{i}|^2} (q_{1} -1)
\end{array}
\right] \\
=& \ q_{i}^2 + \frac{q_2^2 q_i^2}{(\sum_{i=2}^{n} |q_{i}|^2)^2} (q_{1} -1)^2 + \cdots + \frac{q_i^2 q_i^2}{(\sum_{i=2}^{n} |q_{i}|^2)^2} (q_{1} -1)^2  + \frac{2q_i q_i}{\sum_{i=2}^{n} |q_{i}|^2} (q_{1} -1) + 1 \\
&+ \frac{q_n^n q_i^2}{(\sum_{i=2}^{n} |q_{i}|^2)^2} (q_{1} -1)^2 \\
=& \ q_{i}^2 + \frac{q_i^2 (q_{1} -1)^2}{\sum_{i=2}^{n} |q_{i}|^2} + \frac{2q_i^2(q_{1} -1)}{\sum_{i=2}^{n} |q_{i}|^2} + 1
= q_{i}^2 + \frac{q_i^2 (q_{1}^2 -1)}{\sum_{i=2}^{n} |q_{i}|^2} + 1 \\
=& q_{i}^2 - q_{i}^2 + 1 \\
=& \ 1 \\
\end{aligned}
\end{equation}

\begin{equation}
\label{eq:case4}
\begin{aligned}
\bm{u_i}^\top \bm{u_j} =&
\left[
\begin{array}{c}
-q_{i} \\ 
F(2, i) \\
\vdots \\
F(i, i) + 1 \\
\vdots \\
\vdots \\ 
F(n, i)
\end{array}
\right]^\top
\left[
\begin{array}{c}
-q_{j} \\ 
F(2, j) \\
\vdots \\
\vdots \\
F(j, j) + 1 \\
\vdots \\ 
F(n, j)
\end{array}
\right] 
=
\left[
\begin{array}{c}
-q_{i} \\ 
\frac{q_2 q_i}{\sum_{i=2}^{n} |q_{i}|^2} (q_{1} -1) \\
\vdots \\
\frac{q_i q_i}{\sum_{i=2}^{n} |q_{i}|^2} (q_{1} -1) + 1 \\
\vdots \\ 
\vdots \\ 
\frac{q_n q_i}{\sum_{i=2}^{n} |q_{i}|^2} (q_{1} -1)
\end{array}
\right]^\top
\left[
\begin{array}{c}
-q_{j} \\ 
\frac{q_2 q_j}{\sum_{i=2}^{n} |q_{i}|^2} (q_{1} -1) \\
\vdots \\
\vdots \\
\frac{q_j q_j}{\sum_{i=2}^{n} |q_{i}|^2} (q_{1} -1) + 1 \\
\vdots \\ 
\frac{q_n q_j}{\sum_{i=2}^{n} |q_{i}|^2} (q_{1} -1)
\end{array}
\right] \\
=& \ q_{i}q_{j} + \frac{q_2^2 q_i q_j}{(\sum_{i=2}^{n} |q_{i}|^2)^2} (q_{1} -1)^2 + \cdots + \frac{q_i^2 q_i q_j}{(\sum_{i=2}^{n} |q_{i}|^2)^2} (q_{1} -1)^2 + \frac{q_i q_j}{\sum_{i=2}^{n} |q_{i}|^2} (q_{1} -1) \cdots \\
&+ \frac{q_j^2 q_i q_j}{(\sum_{i=2}^{n} |q_{i}|^2)^2} (q_{1} -1)^2 + \frac{q_j q_i}{\sum_{i=2}^{n} |q_{i}|^2} (q_{1} -1) + \cdots + \frac{q_n^2 q_i q_j}{(\sum_{i=2}^{n} |q_{i}|^2)^2} (q_{1} -1)^2 \\
=& \ q_{i}q_{j} + \frac{q_i q_j}{\sum_{i=2}^{n} |q_{i}|^2}(q_{1} -1)^2 + \frac{q_i q_j}{\sum_{i=2}^{n} |q_{i}|^2} (q_{1} -1) + \frac{q_j q_i}{\sum_{i=2}^{n} |q_{i}|^2} (q_{1} -1) \\
=& \ q_{i}q_{j} + \frac{q_i q_j (q_{1}^2 -1)}{\sum_{i=2}^{n} |q_{i}|^2} \\
=& \ 0
\end{aligned}
\end{equation}
According to Eqs.~\eqref{eq:case1},~\eqref{eq:case2},~\eqref{eq:case3},~\eqref{eq:case4}, we thus have that the matrix $U$ is orthogonal.
\end{proof}
\end{strip}

\,

\end{document}